\documentclass[lettersize,journal]{IEEEtran}
\usepackage{amsmath,amsfonts}
\usepackage{algorithmic}
\usepackage{array}
\usepackage[caption=false,font=footnotesize,labelfont=rm,textfont=rm]{subfig}
\usepackage{textcomp}
\usepackage{stfloats}
\usepackage{url}
\usepackage{verbatim}
\usepackage{graphicx}
\def\BibTeX{{\rm B\kern-.05em{\sc i\kern-.025em b}\kern-.08em
    T\kern-.1667em\lower.7ex\hbox{E}\kern-.125emX}}
\usepackage{balance}
\usepackage{cite}
\usepackage[dvipsnames, svgnames, x11names]{xcolor}
\usepackage{xspace}
\usepackage{amsthm}
\usepackage{makecell}
\usepackage{multirow}
\usepackage[ruled]{algorithm}

\newcommand{\iss}{\mathrm{ViTScore}}
\newcommand{\ours}{\texttt{ViTScore}\xspace}
\newcommand{\ms}{\rm{MS-SSIM}\xspace}
\newcommand{\lp}{\rm{LPIPS}\xspace}
\newcommand{\psnr}{\rm{PSNR}\xspace}

\newtheorem{definition}{Definition}
\newtheorem{theorem}{Theorem}
\newcommand{\vit}{\mathrm{ViT}}

\newcommand{\ie}{\textit{i.e.,~}}

\begin{document}
\title{How to Evaluate Semantic Communications for Images with \ours Metric?}

\author{Tingting~Zhu,
Bo~Peng,
Jifan~Liang,
Tingchen~Han,
Hai~Wan{*},
Jingqiao~Fu,
and~Junjie~Chen
\thanks{The authors are with the School of Computer Science and Engineering and the Guangdong Key Laboratory of Information Security Technology, Sun Yat-sen University, Guangzhou 510006, China (e-mail: \{zhutt, pengb8, liangjf56, hantch\}@mail2.sysu.edu.cn, \{wanhai, fujq3\}@mail.sysu.edu.cn, chenjj376@mail2.sysu.edu.cn). \\
\textit{*Corresponding author: Hai Wan.}}
}

\markboth{Journal of \LaTeX\ Class Files.}
{Zhu \MakeLowercase{\textit{et al.}}: How to Evaluate Semantic Communications for Images with \ours Metric}

\maketitle

\begin{abstract}
  \noindent \emph{Semantic communications}~(SC) have been expected to be a new paradigm shifting to catalyze the next generation communication, whose main concerns shift from accurate bit transmission to effective semantic information exchange in communications. However, the previous and widely-used metrics for images are not applicable to evaluate the image semantic similarity in SC. Classical metrics to measure the similarity between two images usually rely on the pixel level or the structural level, such as the PSNR and the MS-SSIM. Straightforwardly using some tailored metrics based on deep-learning methods in CV community, such as the LPIPS, is infeasible for SC. To tackle this, inspired by BERTScore in NLP community, we propose a novel metric for evaluating image semantic similarity, named \emph{Vision Transformer Score}~(ViTScore). We prove theoretically that ViTScore has 3 important properties, including \emph{symmetry, boundedness,} and \emph{normalization}, which make ViTScore convenient and intuitive for image measurement. To evaluate the performance of ViTScore, we compare ViTScore with 3 typical metrics~(PSNR, MS-SSIM, and LPIPS) through 4 classes of experiments: (i) correlation with BERTScore through evaluation of image caption downstream CV task, (ii) evaluation in classical image communications, (iii) evaluation in image semantic communication systems, and (iv) evaluation in image semantic communication systems with semantic attack. Experimental results demonstrate that ViTScore is robust and efficient in evaluating the semantic similarity of images. Particularly,  ViTScore outperforms the other 3 typical metrics in evaluating the image semantic changes by semantic attack, such as image inverse with Generative Adversarial Networks~(GANs). This indicates that ViTScore is an effective performance metric when deployed in SC scenarios. Furthermore, we provide an ablation study of ViTScore to demonstrate the structure of the ViTScore metric is valid.
  \looseness=-1
\end{abstract}


\section{Introduction}

\IEEEPARstart{H}{igher} bit-level precision and fewer bit-level transmission cost are the core concerns of classical transmission. Recently, a flurry of research rethinks these two major interests and investigates the concept and potential of semantics in modern and/or future communications. Shannon and Weaver (1949) once categorized communications into three levels: transmission of information symbols, semantic exchange, and effectiveness of semantic exchange~\cite{shannon1949mathematical}. This indicates that transmitting the semantics, instead of the bits or symbols, may achieve higher system efficiency. 
As presented in~\cite{shannon1949mathematical, Lu2022RethinkingSem,hoydis2021toward, tong2022nine}, {\em semantic communications}~(SC) focus on the transmission of semantic features and the performance improvement at the semantic level, which is expected to be a new paradigm shifting to catalyze the next generation communication. 
SC have attracted growing interest from both academia and industry~(see~\cite{tong2022nine,gunduz2022beyond,sana2022learning,farshbafan2022common,Bockelmann,SongKangXu,JankowskiMikolaj,hu2022robust,liu2021semantics,dai2022nonlinear,zhang2023predictive,Jialong, Bourtsoulatze, XieHuiqiang,huang2021deep,xie2021task,Huiqiang2022Qin,guler2018semantic,zhang2022context,yao2022semantic,seo2023semantics} and the references therein).
However, the pioneering works on SC generally focus on the design of new schemes for SC systems.
Till now, the related works concerning the performance evaluation metric for SC are still in their infancy. Especially, the lack of desirable evaluation metrics for image semantics is highlighted.
\looseness=-1

In traditional image transmission systems, the most convincing metric used to evaluate image information is based on the bit or pixel level, which encourages an exact recurrence of what has been sent from the transmitter~\cite{Lu2022RethinkingSem}.
Classical metrics, such as \textit{mean square error}~(MSE) and \textit{Peak Signal-to-Noise Ratio}~(PSNR), usually take into account the bit (or symbol) or pixel error rate. 
These approaches treat each bit or pixel as equally important, which is in fact unnecessary.
At the same time, they cannot describe the human perceptual judgments of images and fulfill the requirements of desirable intelligent task performance~\cite{li2022region}.
Moreover, the {\em Structural Similarity Index Metric}~(SSIM) and the Multi-Scale SSIM~(MS-SSIM) are sensitive in measuring variability at the level of image luminance, contrast and structure, but they are also not precise in measuring image high-level semantics. 
Thus, it is essential to exploit perceptual-centric evaluation metrics, especially for high inference accuracy in the {\textit{artiﬁcial intelligence}}~(AI) technology based SC scenarios. As the visual similarity~(also referred to as perceptual similarity) could be measured by the distance in feature space~\cite{blau2018perception}, it is necessary for the metric's ability to accurately represent the semantic feature of images.
However, directly using the metrics based on deep-learning~(DL-based) methods in {\em computer vision}~(CV) community, such as the {\em Learned Perceptual Image Patch Similarity}~(LPIPS), is infeasible for image SC~\cite{zhang2018unreasonable}.
Since LPIPS uses the mean pooling of feature distances from corresponding locations to evaluate the similarity of two images. 
As the features are from a bunch of convolution layers, the constraint of corresponding positions leads to the absence of global semantic measures. 
Motivated by these issues, we focus on the SC evaluation metric for images global semantic similarity in this paper.
\looseness=-1

\begin{figure*}[htbp]
  \centering
  \subfloat[Original image.]{
  \includegraphics[width=0.25\textwidth]{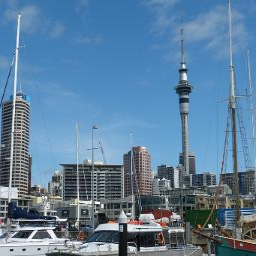}
  }
  \quad
  \subfloat[Reconstructed image with GANs-based SC system model.]{
  \includegraphics[width=0.25\textwidth]{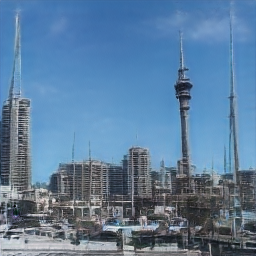}
  }
  \quad
  \subfloat[Reconstructed image with GANs-based SC system model under semantic noise interference.]{
  \includegraphics[width=0.25\textwidth]{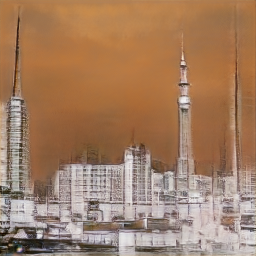}
  }
  \caption{Comparison of the image semantic similarity evaluation with 4 metrics~(\psnr, \ms, \lp and \ours\protect\footnotemark[1]).
  Transmitting an image (a) with a GANs-based SC system model, where the physical channel is supposed to be noiseless, the reconstruction performance (b) is evaluated: \psnr = 14.60, \ms = 0.72, \lp = 0.15, \ours = 0.79.
  While the semantic noise is introduced, which is characterized as the image inverse with GANs, the reconstruction performance (c) is evaluated: \psnr = 5.19, \ms = 0.00, \lp = 0.43, \ours= 0.60.
  The fluctuation in terms of the other 3 metrics is much larger than that of \ours. In fact, the semantics between images (a) and (b) (resp. (a) and (c)) is very similar. Hence, the semantic similarities between images (a) and (b) versus those of (a) and (c) are supposed to be close.}
  \label{ImgTrans}
\end{figure*}

We illustrate the evaluation performance of these image metrics with a typical example. 
It is stated that semantic noise can be classified into two types, referred to as semantic ambiguity and semantic noise coming from the adversarial examples~\cite{Qin2022SemanticCP}. 
We take the latter one as an example. 
It has shown that a variety of semantics emerges in the latent space of Generative Adversarial Networks~(GANs). Some existing works on SC~\cite{hu2023robust,han2023generative,du2023rethinking,wang2022perceptual} have exploited this property of GANs to develop the SC schemes.
We utilize a GAN-based SC system model, which will be described in detail in Section IV-D. For simplicity to the discussion of the semantic changes, the physical channel in this model is supposed to be noiseless.
Transmitted by the GAN-based SC system model with and without semantic attack, which is modeled as image inverse, an image may degrade a lot at the pixel level and the structural level, even with the DL-based metric \lp. However, it may still keep its semantics in significant measure. See Fig. \ref{ImgTrans} for reference.
From this motivation example, it is clear that a general metric for image semantic similarity is missing. 
Even the typical DL-based metric \lp cannot precisely measure the global semantics of images. To tackle this, we focus on presenting an evaluation metric that can perform well in extracting the global semantic characters of images.

\footnotetext[1]{The details about the Vision Transformer Score~(\ours) proposed in this paper will be described in Section III.}

Inspired by the state-of-the-art metric BERTScore~\cite{DBLP:conf/iclr/ZhangKWWA20} in the {\em natural language processing}~(NLP)  community and pre-training approach in foundation models~\cite{bommasani2021opportunities}, we propose a novel evaluation metric for image SC, named \textit{Vision Transformer Score}~(\ours), which measures the image semantic similarity based on the pre-trained image model \textit{Vision Transformer}~(ViT)~\cite{DBLP:conf/iclr/DosovitskiyB0WZ21}. ViT demonstrates a strong ability to integrate global semantic information thanks to the self-attention mechanism~\cite{DBLP:conf/iclr/DosovitskiyB0WZ21}.
To the best of the authors' knowledge,
\ours is the first practical image semantic similarity evaluation metric based on ViT in SC.
Specifically, our main contributions can be summarized as five folds:
\begin{itemize}
  \item The \ours metric is first proposed to measure the semantic similarity of images based on the pre-trained model ViT. To be specific, by exploiting the attention mechanism, a pre-trained image model ViT is employed in the \ours metric to learn and extract essential semantics of the image.
  \item We define \ours (see \textbf{Definition 2}), and prove theoretically that \ours has 3 important properties, including symmetry, boundedness, and normalization (see \textbf{Theorems 1, 2, and 3}, respectively). These properties make \ours convenient and intuitive for implementation in image measurement.
  \item To evaluate the performance of \ours metric, we compare \ours with 3 typical categories of metrics (\psnr, \ms, and \lp) through 4 classes of experiments: (i) correlation with BERTScore through evaluation of image caption downstream CV task, (ii) evaluation in classical image communications, (iii) evaluation in image semantic communication systems~(including 4 semantic communication models~(NTSCC~\cite{dai2022nonlinear}, DeepJSCC-V~\cite{zhang2023predictive}, ADJSCC~\cite{Jialong} and DeepJSCC~\cite{Bourtsoulatze})~\footnote[2]{The details about these 4 models will be described in Section IV. C.}) and (iv) evaluation in image semantic communication systems with semantic attack~(including 2 categories of semantic attack: semantic attack to mislead image classification task and semantic attack with image transforms (such as image inverse)). Experimental evaluations show that \ours is robust and well-performed. 
  \item Particularly, we evaluate the \ours performance with different categories of semantic attack in images SC. Numerical results show that \ours can better reflect the image semantic changes than the other 3 typical metrics, indicating that \ours has a promising performance when deployed in image SC scenarios.
  \item We provide an ablation study of \ours to demonstrate the structure of the \ours metric is valid.
\end{itemize}

The rest of this paper is structured as follows. The related works on the image semantic similarity metrics are discussed in Section II. Section III introduces the proposed \ours metric in detail. In Section IV, the metric evaluations of \ours are analyzed through 4 classes of experiments. Section V draws conclusions and discusses our future work.
\looseness=-1

\section{Related Work}

Semantic delivery of images is essential for 6G communication systems, and hence, image SC has been studied extensively in recent years~\cite{gunduz2022beyond}. However, most of the existing research evaluates the image SC system performance with the traditional image quality metrics, which are not designed to evaluate the semantics of images.
Most of the previous studies on evaluating image quality rely on the pixel level, such as $\ell_2$ distance, MSE and PSNR. These metrics usually treat each pixel of the image as equally important, which is practically unnecessary. On the other hand, the metrics based on pixels cannot well reflect the human perceptual judgments of images.
Therefore, SSIM, a class of metrics to measure images at the structural level, has emerged. There are many variants of SSIM, such as MS-SSIM, Multi-component SSIM, and Complex Wavelet SSIM.
They are sensitive in demonstrating the image luminance, contrast and structure, however, they are not precise in measuring image semantics.
With the prosperity of AI technology, specifically the rise of deep learning methods, some tailored DL-based metrics are developed in the CV community.
Nevertheless, straightforwardly using the DL-based metrics to SC systems is infeasible.
To be specific, we classify the most commonly used image quality assessment metrics into three categories.
\looseness=-1

\subsection{Pixel Level Measures}

MSE is the simplest and most widely used image quality assessment metric. Based on MSE, the PSNR is given by
\begin{equation}
    {\rm PSNR} = 10 \log_{10} \frac{{{\rm MAX}}^2}{{\rm MSE}} ({\rm dB}),
\end{equation}

\noindent where MSE $= d(x,\hat{x})$ is the mean square error between the reference image $x$ and the reconstructed image $\hat{x}$, and MAX is the maximum possible value of the image pixels\footnote[3]{All our experiments measured in terms of PSNR in this paper are conducted on 24-bit depth RGB images (8 bits per pixel per color channel), thus MAX $= 2^8 - 1 = 255$.}. 
In general, the larger the \psnr, the more similar the two images.
\psnr is an appealing metric due to that it is simple to calculate, has clear physical meanings, and is mathematically convenient in the context of optimization. But it is not very well matched to perceived visual quality~\cite{wang2004image,girod1993s,eskicioglu1995image,wang2002universal,wang2002image}.

\subsection{Structural Level Measures}

Unlike MSE and PSNR, which measure absolute error, a class of metrics focusing on human intuitive perception is presented, named SSIM.
SSIM mainly considers 3 key features of an image: luminance, contrast, and structure.
For two images $A$ and $B$, the SSIM between them is defined as

\begin{gather}
    \text{Luminance:} ~~~ l(A,B)= \frac{2 \mu_A \mu_B + c_1}{\mu_A^2 + \mu_B^2 + c_1}, \\
    \text{Contrast:} ~~~ c(A,B)= \frac{2 \sigma_A \sigma_B + c_2}{\sigma_A^2 + \sigma_B^2 + c_2},  \\
    \text{Structure:} ~~~ s(A,B)= \frac{\sigma_{AB} + c_3}{\sigma_A\sigma_B + c_3}, \\
    \text{SSIM:} ~~~ {\rm SSIM}(A,B)= [l(A,B)]^{\alpha}[c(A,B)]^{\beta}[s(A,B)]^{\gamma},
\end{gather}

\noindent where $c_i, i\in\{1,2,3\}, \alpha, \beta$ and $\gamma$ are positive constants. $\mu$ and $\sigma^2$ denote the mean intensity and variance of an image, respectively. Essentially, SSIM is a single-scale algorithm. The exact image scale depends on the viewing conditions of the user, such as the resolution of the display device, the viewing distance of the user, etc.
Therefore, the \ms algorithm has emerged to improve evaluation performance in good agreement with human perceptual judgments.
The \ms is given by
\begin{equation}
    {\rm MS\text{-}SSIM} (A,B)=  [l_M(A,B)]^{\alpha_M} \cdot \prod_{j=1}^{M}[c_j(A,B)]^{\beta_j}[s_j(A,B)]^{\gamma_j}, 
\end{equation}

\noindent where $M$ is the maximum scale of the image. For the $j$-scale, the similarity in luminance, contrast and structure of $A$ and $B$ are denoted by $l_j(A,B)$, $c_j(A,B)$, and $s_j(A,B)$, respectively.
The \ms~in dB is given by
\begin{equation}
    {\rm MS\text{-}SSIM} = - 10 \log_{10} (1-  {\rm MS\text{-}SSIM} ) ({\rm dB}).
\end{equation}

SSIM / \ms is between $0$ and $1$. Generally, the larger the SSIM / \ms, the smaller the difference between the two images at the structural level. Particularly, SSIM / \ms equals 1 if and only if the two images are exactly the same. This class of metrics well performs in measuring the image luminance, contrast and structure, but fails to evaluate for image rotation, translation, and other geometric operations.

\subsection{DL-based Measures}

Benefiting from the deep learning methods, several DL-based metrics for image quality measurement have been developed.
In~\cite{zhang2018unreasonable}, the authors show the unreasonable effectiveness of deep features as a perceptual metric, indicating that classification networks perform better than low-level metrics, such as $\ell_2$, \psnr and \ms, in evaluating the images' perceptual similarity.
\lp is the most popular one among perceptual metrics. Given an image $X$ and a deep neural network $\mathcal F$ with $l$ layers which takes $X$ as its input. The output of $i$-th layer is denoted as $Y^i_X\in\mathbb R^{H_i\times W_i\times C_i}$, where $H_i$, $W_i$ and $C_i$ are height, weight and number of channels of the $i$-th layer, respectively. Given two images $A$ and $B$, the \lp score of them is defined as
\begin{equation}
    \label{eq:LPIPS}
    \mathrm{LPIPS}(A, B) = \sum_i\frac{1}{H_iW_i}\sum_{h,w}||\mathbf{\theta^i}\odot(Y^i_{A(h, w)}-Y^i_{B(h, w)})||_2^2
\end{equation}

\noindent where $\mathbf{\theta^i}\in\mathbb R^{C_i}$ are learned parameters, and $\odot$ is the element-wise production.
Typically $\mathcal F$ is a pre-trained deep neural network with multiple convolution layers, and \lp needs to be further trained ($\mathbf{\theta^i}$ and/or other parameters in $\mathcal F$).
In general, the lower the LPIPS, the more similar the two images.
\lp uses the mean pooling of feature distances from corresponding locations to evaluate the similarity of two images, leading to an absence of global semantic measures.

\begin{table*}[!t]
  \renewcommand\arraystretch{1.5}
  \caption{\label{pros-cons} Comparison of image quality metrics within 3 typical categories.}
  \setlength{\tabcolsep}{2mm}{
    \centering
  \begin{tabular}{|c|c|c|c|}
  \hline
  {Categories} & {Metrics} & {Advantages} & {Limitations} \\ \hline \hline 
  {\makecell[c]{Pixel level\\measures}} & {\makecell[c]{$\ell_2$ distance, \\MSE, \\RMSE,\\MAE,\\PSNR}} & {\makecell[c]{1. Simple to calculate and widely used.\\2. Well perform in the MSE measurement at the pixel level.\\3. Has clear physical meanings.\\4. Mathematically convenient in the context\\of optimization~\cite{wang2004image}.}} & {\makecell[c]{1. Usually computed over all pixels in both images.\\2. Cannot reflect the structural and semantic\\information well.\\3. Poor correlation with human subjective quality\\ratings, unable to differentiate image contents~\cite{liu2017free}.\\4. Failed to match the perceived visual quality~\cite{wang2004image}.}}  \\ \hline \hline 
  {\makecell[c]{Structural level\\measures}} & {\makecell[c]{SSIM~\cite{wang2004image},\\MSSIM~\cite{wang2004image},\\MS-SSIM}} & {\makecell[c]{1. Sensitive in measuring variability at the level of\\image luminance, contrast and structure.\\2. Has the properties of symmetry, boundedness,\\and unique maximum.\\3. Used not only to evaluate but also to optimize a large\\variety of signal processing algorithms and systems~\cite{wang2009mean}.}} & {\makecell[c]{1. Not precise in measuring image semantics.\\2. Can not perform well with image rotation,\\translation, and other geometric operations.}}  \\ \hline \hline 
  {\makecell[c]{DL-based\\measures}} & {\makecell[c]{LPIPS~\cite{zhang2018unreasonable}, \\VTransE~\cite{zhang2017visual},\\FID~\cite{heusel2017gans}}} & {\makecell[c]{1. Well performs in human perceptual\\similarity judgments.\\2. Widely used in image generation, restoration,\\enhancement and super-resolution.}} & {\makecell[c]{1. Fails to precisely measure the\\global semantics of images.\\2. Depends on the selection of the training dataset.}} \\ \hline 
  \end{tabular}}
\end{table*}

Table~\ref{pros-cons} describes the advantages and limitations of these 3 typical categories of image quality metrics in detail.
We can conclude that the existing typical image quality metrics cannot meet the SC system performance evaluation requirements.
Inspired by the capacity of AI technology to process semantics, especially the advances in NLP, researchers keep a watchful eye on techniques of textual similarity metrics. Among them, the metric \textit{bilingual evaluation understudy} (BLEU) score~\cite{papineni2002bleu} is a widely used one to evaluate the quality of sentences that are translated by machine. The value of the BLEU score is a number ranging from 0 to 1. The higher the score, the better the quality of translation, i.e. the higher the similarity between two sentences. Whereas, the BLEU score only focuses on the difference between words in two sentences rather than the semantic information similarity of two sentences. To overcome this issue, a metric, named sentence similarity, has been proposed in~\cite{xie2021deep}. The metric sentence similarity is a measure closer to human judgment, defined as
\begin{equation}
	{\rm match} ( \hat{\textbf{s}}, \textbf{s}) = \frac{\textit{\textbf{B}}_{\Phi}(\textbf{s}) \cdot \textit{\textbf{B}}_{\Phi}(\hat{\textbf{s}})^{T}}{\parallel \textit{\textbf{B}}_{\Phi}(\textbf{s})\parallel \parallel \textit{\textbf{B}}_{\Phi}(\hat{\textbf{s}}) \parallel}
\end{equation}

\noindent where $\textbf{s}$ is the transmitted sentence, and $\hat{\textbf{s}}$ is the reconstructed sentence, $\textit{\textbf{B}}_{\Phi}$ is the \textit{ Bidirectional Encoder Representations from Transformers}~(BERT) based on a pre-trained language representation model with huge parameters \cite{devlin2018bert}. These approaches show robust performance and achieve a higher correlation with human judgment than that of other previous metrics in many NLP tasks.
\looseness=-1

Enlightened by the boom in pre-trained models and the advantages of metrics in the NLP community, the development of semantic similarity metrics of images is encouraging.
In the existing investigations of SC, only a handful of papers have studied the performance metrics, especially for images.
Among the few papers that consider image transmission, most proposed SC systems are based on DL (see~\cite{huang2022toward,johnson2016perceptual,wang2014learning,frome2013devise,zhang2017visual,zhang2019large} and the references therein), while none of them presents a practical semantic metric for image similarity measurement. Most existing studies on image SC systems use the \psnr, \ms or \lp as metrics to evaluate the performance of the system, which do not indeed evaluate the semantic information of the performance of the system.
Motivated by this, we focus on developing a novel metric based on a pre-trained transformer model for the SC of images in this paper.
\looseness=-1

\section{\ours Metric}

\begin{figure*}[tb]
  \centering
  \includegraphics[width=0.65\textwidth]{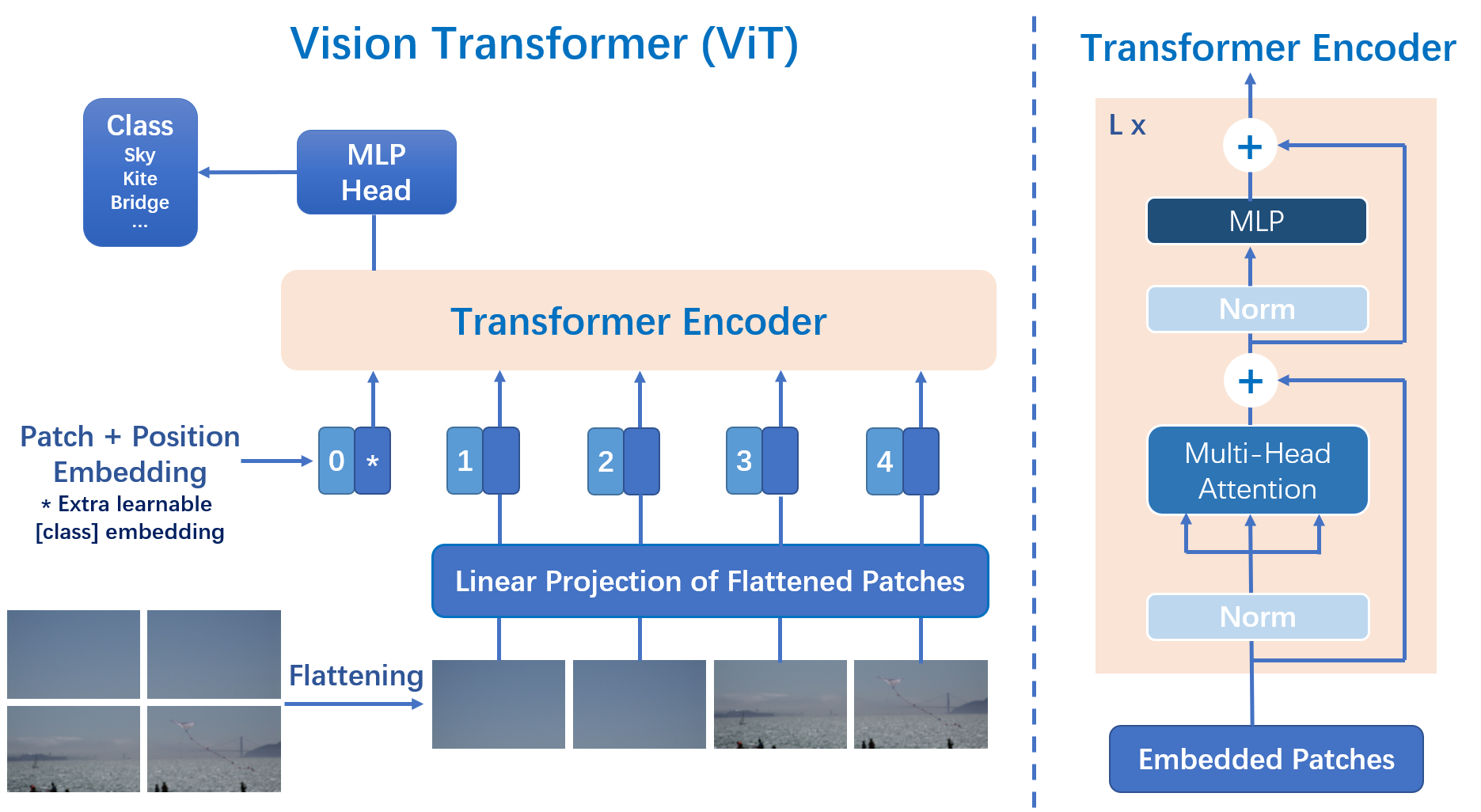}
  \caption{\label{fig:vit}The model overview of ViT\cite{DBLP:conf/iclr/DosovitskiyB0WZ21}. The input image is first split into patches with a fixed size. Then the patches are flattened and projected to the embedding space by a linear layer. The position embedding is then added, in order to keep the positional information of patches. The mixed embedding vectors are fed into an L-layer transformer encoder. Each layer of the transformer encoder is shown on the right side of this figure. The outputs of the transformer encoder are semantic features and can be further fed into an MLP classifier. As we use ViT to extract features, our implementation does not consist of the MLP head and the classification output.}
\end{figure*}

Inspired by BERTScore\cite{DBLP:conf/iclr/ZhangKWWA20}, we propose \ours to evaluate the semantic similarity between two images. \ours uses pre-trained ViT to extract the semantic features of given images, then calculates the semantic similarity using these features.
Formally, referring to~\cite{DBLP:conf/iclr/DosovitskiyB0WZ21}, we give an explicit definition of ViT.

\begin{definition}[ViT]
  \label{def:vit}
  Given an image with width of $W$, height of $H$ and $C$ channels, the ViT extracts its semantic features in an $\frac{HW}{P^2} \times N$ matrix:
  \begin{equation}
    \vit: \mathbb R^{H\times W\times C}\to\mathbb R^{\frac{HW}{P^2}\times N},
  \end{equation}
  where $\mathbb R$ is the set of real numbers, $P$ and $N$ are hyperparameters of ViT. 
  By ViT, the image is first split into multiple patches, each of which has a height of $P$ and a width of $P$. 
  Therefore the number of patches is $\frac{HW}{P^2}$.
  Then, each patch is extracted as a real vector of length $N$. 
  The model overview is shown in Fig.~\ref{fig:vit}.
  Given an image $A$ with $n$ patches, we denote $(\mathbf{a_0}, \mathbf{a_1}, \dots, \mathbf{a_{n-1}}) = \vit(A)$ as the semantic features extracted from ViT, where $\mathbf{a_i}\in\mathbb R^N$. The vectors are $\ell_2$ normalized, \ie $||\mathbf{a_i}||_2^2=1$.
  \looseness=-1
\end{definition}

ViT is proposed as a pre-trained deep neural model for image classification. It is pre-trained with large image sets to understand the image semantics. Then it can be further fine-tuned for various downstream classification tasks. Since we use ViT to extract semantic features, we do not fine-tune ViT and use the pre-trained model parameters directly.

\subsection{Definition of \ours}
Based on the definition of ViT, we define \ours between two images as follows.

\begin{definition}[ViTScore]
  \label{def:VitScore}
  Given two images: $A$ with $n$ patches and $B$ with $m$ patches. To evaluate the semantic similarity between them, we first extract their semantic features using $\vit$ defined in Definition \ref{def:vit}:
  \begin{equation}
    \begin{split}
      \vit(A) &= (\mathbf{a_0}, \mathbf{a_1}, \dots, \mathbf{a_{n-1}})  \\
      \vit(B) &= (\mathbf{b_0}, \mathbf{b_1}, \dots, \mathbf{b_{m-1}})  \\
    \end{split}
  \end{equation}

  Then we calculate recall $R_\iss(A, B)$ and precision $P_\iss(A, B)$:
  \begin{equation}
    \label{eq:PR}
    \begin{split}
      R_\iss(A, B)&=\frac{1}{n}\sum_{i=0}^{n-1}\max_{0 \leq j < m}\mathbf a_i^\top\mathbf b_j\\
      P_\iss(A, B)&=\frac{1}{m}\sum_{j=0}^{m-1}\max_{0 \leq i < n}\mathbf a_i^\top\mathbf b_j
    \end{split}
  \end{equation}

  Finally, we calculate the \ours:
  \begin{equation}
    \label{eq:vitscore}
    \iss(A, B)=2\frac{R_\iss(A, B)\cdot P_\iss(A, B)}{R_\iss(A, B) + P_\iss(A, B)} .
  \end{equation}
\end{definition}

\subsection{Properties of \ours}

\ours has 3 important properties, including symmetry, boundedness, and normalization. We will introduce the details and provide proof in the following.

\begin{theorem}[Symmetry]
  For any two images $A$ and $B$, changing the order of the two images does not change the $\iss$, \ie
  \begin{equation}
    \iss(A, B) = \iss(B, A).
  \end{equation}
\end{theorem}

\begin{proof}
According to Equation (\ref{eq:PR}) we have

\begin{equation}
      R_\iss(A, B) = \frac{1}{n}\sum_{i=0}^{n-1}\max_{0 \leq j < m}\mathbf a_i^\top\mathbf b_j
      = P_\iss(B, A).
\end{equation}

Therefore according to Equation (\ref{eq:vitscore}), we have
\begin{equation}
    \begin{split}
        \iss(A, B) &= 2\frac{R_\iss(A, B)\cdot P_\iss(A, B)}{R_\iss(A, B) + P_\iss(A, B)}\\
        &= 2\frac{P_\iss(B, A)\cdot R_\iss(B, A)}{P_\iss(B, A) + R_\iss(B, A)}\\
        &= \iss(B, A).
    \end{split}
\end{equation}
\end{proof}

Since the semantic features extracted from ViT are normalized vectors, $\mathbf{a_i}^\top\mathbf{b_j}$ is the \emph{cosine similarity} between $\mathbf{a_i}$ and $\mathbf{b_j}$. Then, the \ours shares the same boundary with cosine similarity.
\looseness=-1

\begin{theorem}[Boundedness]
  \label{th:range}
  For any two images $A$ and $B$, the range of \ours is between $-1$ and $1$, \ie
  \begin{equation}
    -1\leq\iss(A, B)\leq 1.
  \end{equation}
\end{theorem}

As we applied \emph{greedy match} in our method, \ie the $\max$ operator in Equation (\ref{eq:PR}), the practical lower bound is higher than in Theorem \ref{th:range}. 
Practically, it is hard to find two images that do not have any semantic similarities. Therefore, we evaluate the \ours between a normal image and some random noise to get a practical lower bound of our \ours. In our overall experiments over 6 representative image datasets~(see Table~\ref{universality} for reference), we found the minimum of \ours evaluation is 0.19.
In general, the larger the \ours, the more semantically similar the two images.

\begin{theorem}[Normalization]
  The \ours between any image $A$ and itself is 1, \ie
  \begin{equation}
     \iss(A, A)=1.
  \end{equation}
\end{theorem}

\begin{proof}
The semantic features extracted from ViT are normalized vectors, therefore
\begin{equation}
    \max_{0\leq j < m}\mathbf{a_i}^\top\mathbf{a_j} = \mathbf{a_i}^\top\mathbf{a_i} = ||\mathbf{a_i}||_2^2 = 1
\end{equation}

According to Equation (\ref{eq:PR}), we have
\begin{equation}
    R_\iss(A, A) = P_\iss(A, A) = 1.
\end{equation}

Then according to Equation (\ref{eq:vitscore}), we have
\begin{equation}
    \iss(A, A) = 1.
\end{equation}
\end{proof}

Theorems 1, 2, and 3 prove that \ours has the important properties of symmetry, boundedness, and normalization.
These important properties make \ours convenient and intuitive for implementation in image measurement. 
Table.~\ref{tab:properties} shows the comparison of these properties of \ours to the other 3 typical metrics~(\psnr, \ms, and \lp).

\begin{table}[htb]
  \renewcommand{\arraystretch}{1.5}
  \setlength{\tabcolsep}{5pt}
  \centering
  \caption{\label{tab:properties}Comparison of 3 important properties of \ours to 3 typical image metrics.}
  \begin{tabular}{|c|c|c|c|}
      \hline
       Metrics & Symmetry & Boundedness & Normalization  \\
       \hline \hline
       \ours & $\surd$  & $\surd$ & $\surd$ \\
       \hline
       \psnr & $\surd$ & $\times$ & $\times$ \\
       \hline
       \ms & $\surd$  & $\surd$ & $\surd$ \\
       \hline
       \lp & $\surd$  & $\surd$ & $\surd$ \\
       \hline
  \end{tabular}
\end{table}

\subsection{Ablation Study}

An ablation study is designed to investigate whether some structure or feature of the proposed model is valid.
Different metrics may have different features. For example, according to Equation (\ref{eq:LPIPS}), LPIPS uses mean pooling to evaluate the semantic similarity. While \ours uses max pooling. To evaluate the contributions of these features, we design an \emph{ablated} \ours models, named $\ours_{Mean}$, and compare it with the original \ours, named $\ours_{Origin}$, which is described in Definition \ref{def:VitScore}.

Specifically, we replace the max pooling in Equation (\ref{eq:PR}) by mean pooling:
\begin{equation}
    \begin{split}
        R_\mathrm{MeanPooling}(A, B) &= P_\mathrm{MeanPooling}(A, B) \\
        &= \frac{1}{nm}\sum_{i=0}^{n-1}\sum_{j=0}^{m-1}\mathbf a_i^\top\mathbf b_j.
    \end{split}
\end{equation}

Results are shown in Section \ref{sec:ablation}.

\section{\ours Metric Evaluation}

To evaluate the performance of \ours for image semantic similarity, we compare \ours with the 3 typical categories of metrics (\psnr, \ms, and \lp) through 4 classes of experiments over 10 typical datasets from real-world applications and discuss the effects of the \ours metric.
Speciﬁcally, we answer the following 4 research questions.

\begin{itemize}
  \item \textit{Question 1}: Does \ours \textit{outperform} the previous, widely-used metrics, such as \psnr, \ms, and \lp, for evaluating the image semantic similarity?
  \item \textit{Question 2}: 
  How does \ours perform when \textit{applied to classical image communications}?
  \item \textit{Question 3}: How does \ours perform when \textit{applied to image SC systems}? 
  \item \textit{Question 4}: How does \ours perform when \textit{applied to image SC systems with semantic attack}?
\end{itemize}

We answer the above 4 questions through 4 classes of experiments.
For fairness, we use the following experimental settings.
To calculate \ours, we resized the input image to size $224 \times 224 \times 3$ and fed it into the ViT model. 
We build upon PyTorch~\cite{paszke2019pytorch} and the timm library~\cite{timm}\footnote[4]{We use the parameters from the model vit\_base\_patch16\_224 in the timm library. By choosing appropriate parameters that are already publicly available, such as image size and patch size, we do not need to pre-train the ViT model for a specific image dataset.}.
All experiments are performed on a server with an Intel Xeon Gold 6248R 3.0 GHz CPU, five NVIDIA A100-PCIE-40G GPU and 128 GB of memory.

\subsection{Advantages of \ours in Image Semantic Similarity Evaluation}

In CV field, many downstream tasks, such as image captioning, image classification~\cite{cheng2023class} and semantic segmentation~\cite{wang2017semantic}, require a semantic understanding of an image. Therefore, we can evaluate \ours using these CV tasks.
As \ours is built from BERTScore, it is interesting to evaluate the correction between BERTScore and \ours. Hence, we choose the image captioning task for an experiment in this subsection.
COCO is a representative image captioning dataset\cite{DBLP:conf/eccv/LinMBHPRDZ14}\footnote[5]{The COCO dataset is a large-scale object detection, segmentation, keypoint detection, and captioning dataset. We use the image captioning part within it in this paper.} . In this dataset, each image is labeled with a caption in English. The caption is able to describe the semantics of the image. Thus, we can evaluate the semantic similarities between the two images and their captions. The results are shown in Fig.~\ref{vsBertscore}.
Based on this, we calculate the Pearson correlation coefficient~(PCC) between \psnr and BERTScore, the PCC between \ms and BERTScore, the PCC between \lp and BERTScore, and the PCC between $\iss$ and BERTScore, respectively. The PCC is defined by

\begin{equation}
    {\rm PCC}(x,y) = \frac{Cov(x,y)}{\sqrt{Var(x) Var(y)}},
\end{equation}

\noindent where $Cov(\cdot)$ is the covariance function and $Var(\cdot)$ is the variance function.
Numerical results verify that \ours is positively correlated with BERTScore compared to \psnr, \ms, and \lp.
\looseness=-1

\begin{figure}[htbp]
  \centering
  \subfloat[PCC (\psnr, BERTScore)\\= $0.0515$.]{
  \includegraphics[width=0.21\textwidth]{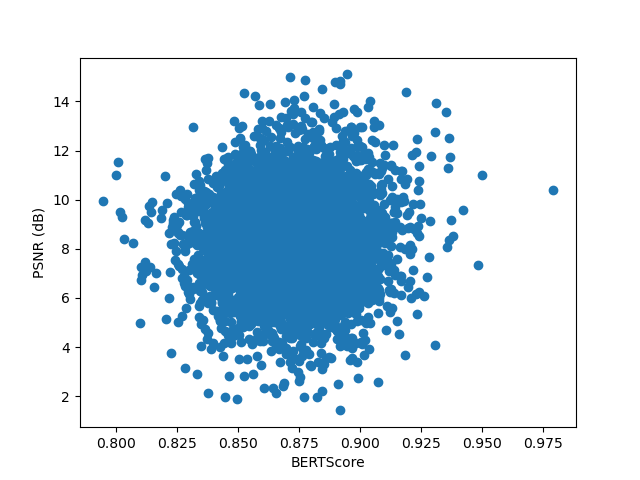}
  }
  \quad
  \subfloat[PCC (\ms, BERTScore)\\= $0.0576$.]{
  \includegraphics[width=0.21\textwidth]{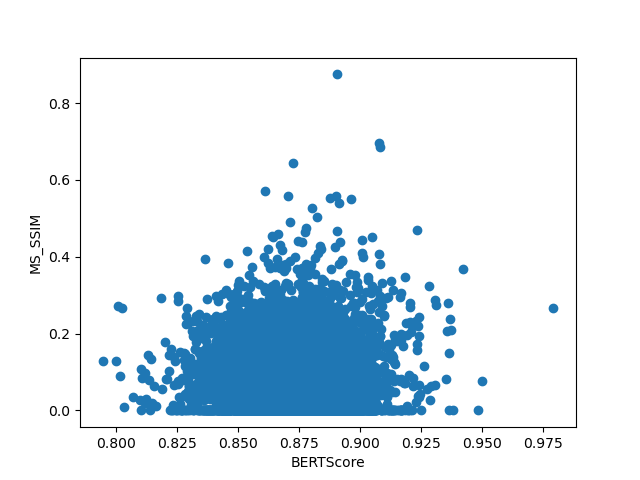}
  }
  \quad
  \subfloat[PCC (\lp, BERTScore)\\= $-0.1217$.]{
  \includegraphics[width=0.21\textwidth]{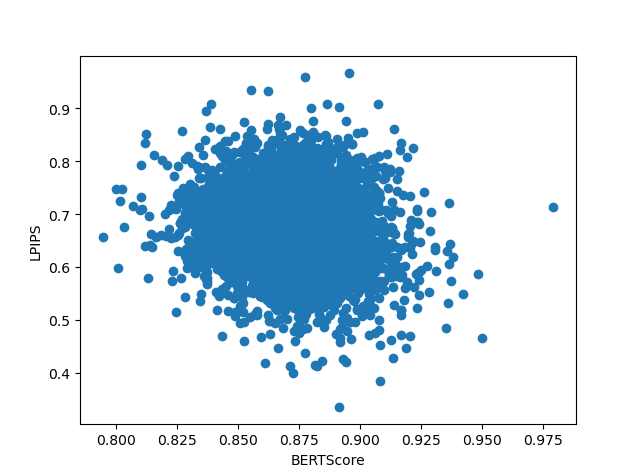}
  }
  \quad
  \subfloat[PCC (ViTScore, BERTScore)\\= $0.2198$.]{
  \includegraphics[width=0.21\textwidth]{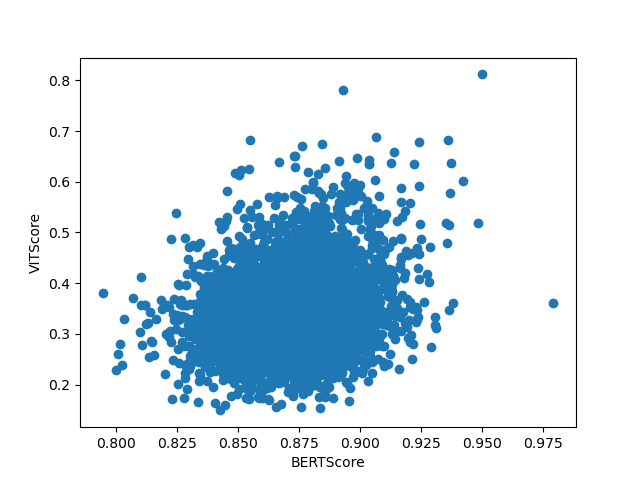}
  }
  \caption{The comparison of the correlations between \psnr, \ms, \lp, \ours and BERTScore over COCO dataset. The negative value of the Pearson correlation coefficient between \lp and BERTScore indicates higher similarity for lower LPIPS values. The Pearson correlation coefficient between \ours and BERTScore is larger than the absolute value of the coefficient between LPIPS and BERTScore, indicating that \ours is more strongly correlated with BERTScore than LPIPS.}
  \label{vsBertscore}
\end{figure}

\begin{figure}[htbp]
  \centering
  \subfloat[Image Caption in COCO dataset: A woman on a tennis court is swinging a racquet.]{
  \includegraphics[width=0.21\textwidth]{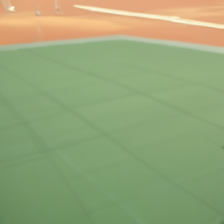}
  }
  \quad
  \subfloat[Image Caption in COCO dataset: A painting of kids on the bathroom floor which is tile.]{
  \includegraphics[width=0.21\textwidth]{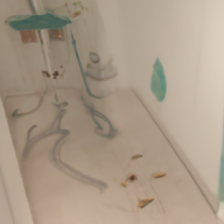}
  }
  \quad
  \subfloat[Image Caption in COCO dataset: Beef and vegetables on a plate sitting on a table.]{
  \includegraphics[width=0.21\textwidth]{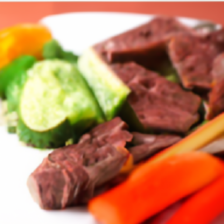}
  }
  \quad
  \subfloat[Image Caption in COCO dataset: A white plate filled with meat and vegetables.]{
  \includegraphics[width=0.21\textwidth]{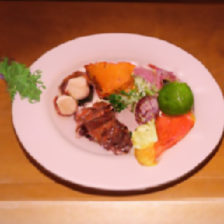}
  }
  \caption{Comparison of the evaluation examples with BERTScore, \psnr, \ms, \lp and \ours. The evaluations between (a) and (b): BERTScore = 0.89, \psnr = 16.23, \ms = 0.47, \lp = 0.56, \ours = 0.32. While the evaluations between (c) and (d): BERTScore = 0.93, \psnr = 7.28, \ms = 0.00, \lp = 0.61, \ours = 0.66. Intuitively, the semantic similarity between (c) and (d) is higher than that of (a) and (b). Surprisingly, \ours performs better than \psnr, \ms and \lp in these cases, which is consistent with BERTScore, agreeing well with human perceptual judgments.}
  \label{ExpvsBS}
\end{figure}

To be specific, we take for example two pairs of images to show the intuitive correlation among the 5 metrics, see Fig.~\ref{ExpvsBS} for reference.
The first pair of images (a) and (b) in Fig.~\ref{ExpvsBS} has low semantic similarity. Both the BERTScore and \ours of them are at low levels.
The second pair of images (c) and (d) in Fig.~\ref{ExpvsBS} has much higher semantic similarity than the first pair. Consequently, the BERTScore and \ours of the second one both result in a relatively high score. Whereas, the \psnr, \ms, and \lp seem to lose their efficacy of semantic similarity measurement in this situation.

\subsection{Applied to Classical Image Communication Systems}

In this subsection, we apply the \ours metric to classical image communication systems. 
We use a source-channel separation-based architecture, combining JPEG source coding and an assuming capacity-achieving channel code.
We carry out numerical experiments over additive white Gaussian noise~(AWGN) channels and Rayleigh fading channels respectively to simulate real wireless image communication scenarios. The capacity of an AWGN channel is given by ${\rm Cap} = \frac{1}{2} \log_2 (1 + 10^{0.1 \times {\rm SNR }})$, where SNR is the average signal-to-noise ratio ${\rm SNR} = 10 \log_{10} \frac{1}{\sigma^2}$ (dB), representing the ratio of the average power of the channel input to the average noise power. Correspondingly, the capacity of a Rayleigh fading channel is given by ${\rm Cap_{Fading}} = \frac{1}{2} \log_2 (1 + 10^{0.1 \times {\rm SNR*h }})$, where $h$ denotes the channel gain that remains constant throughout $k$ channel uses.
In this communication architecture, an input image is first represented as a vector of pixel intensities $\mathbf{x} \in \mathbb{R}^n$, then transmitted in $k$ uses of the channel.
The image dimension $n$ is referred to as the {\textit{source bandwidth}} and is given by the product of the image's height $H$, width $W$ and the number of color channels $C$, {\ie} $n = H \times W \times C$.
Utilizing JPEG coding, we define the {\rm image compression ratio} as $R_{\rm comp} = m/n$, where $m$ is denoted as the size in bits of the compressed image, which is the channel input.
The number of channel-use $k$ is also defined as the {\textit{channel bandwidth}}.
Hence, we define the {\textit{transmission ratio}} (also named as the {\textit{channel bandwidth ratio})} as ${\rm CBR} = k/n$.
\looseness=-1

We compare the image transmission performance in terms of the average \ours with 3 typical metrics (\psnr, \ms, and \lp) over COCO dataset. See Figs.~\ref{AWGN-ViT-JPEGcapacity} and~\ref{Fading-ViT-JPEGcapacity} for details.
Experimental results show that the \ours is universal and robust for image transmission over various channels with a wide range of {\rm CBRs}.
Besides, we observe that the performance trends of \ours are in line with those of \psnr, \ms, and \lp. 
From Figs.~\ref{AWGN-ViT-JPEGcapacity} and~\ref{Fading-ViT-JPEGcapacity}, we can see that the \ours and \lp evaluations converge faster than the PSNR and MS-SSIM evaluations in the high SNR and CBR regions, where the semantics of the images are very similar.
Note that in this implementation, a capacity achieving code is used as the channel code, which means that the performance can be improved as SNR and CBR increase. In practice, however, the performance converges to a threshold value, which is difficult to improve.   Especially, for bit level, the threshold value sometimes maybe low because of the image quantization or the error propagation caused by separated source-channel coding~(SSCC). In contrast, the proposed metric evaluation can converge to a high level even in the low SNR and CBR regions, as \ours evaluates the images at semantic level. 
Moreover, as we have shown in the above-mentioned experiments that \ours outperforms the \psnr, \ms, and \lp for image semantic similarity, we can conclude that the proposed \ours metric is robust and promising when deployed in SC for images.
\looseness=-1

\begin{figure*}[htbp]
  \centering
  \subfloat[\psnr vs \ours.]{
  \includegraphics[width=0.3\textwidth]{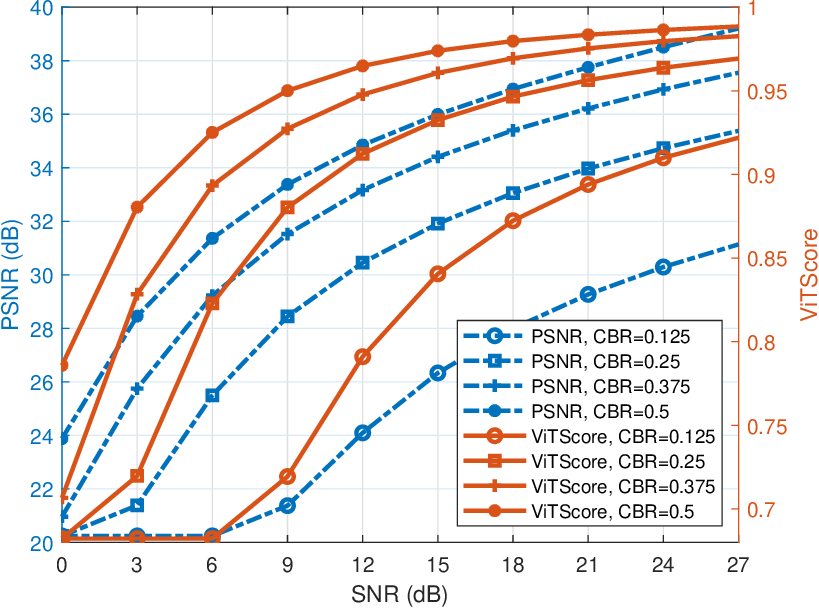}
  }
  \quad
  \subfloat[\ms vs \ours.]{
  \includegraphics[width=0.3\textwidth]{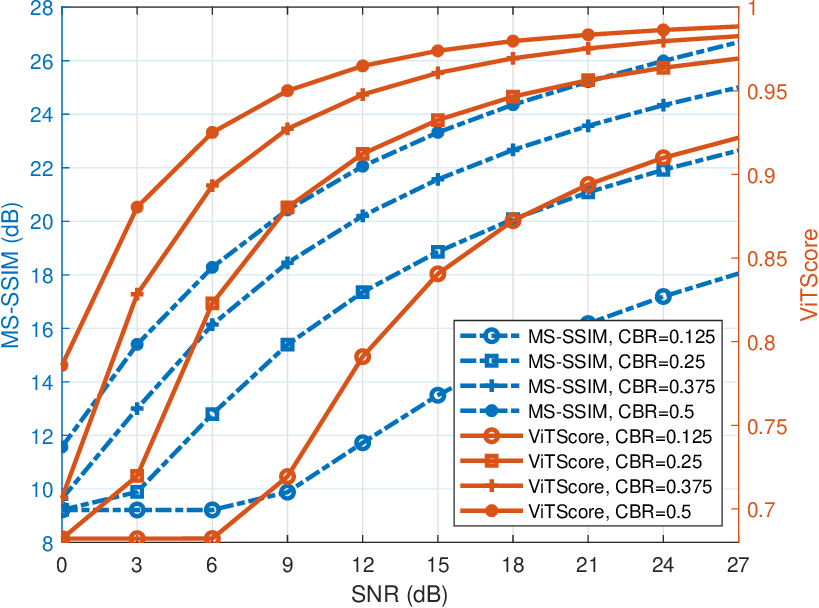}
  }
  \quad
  \subfloat[\lp vs \ours.]{
  \includegraphics[width=0.3\textwidth]{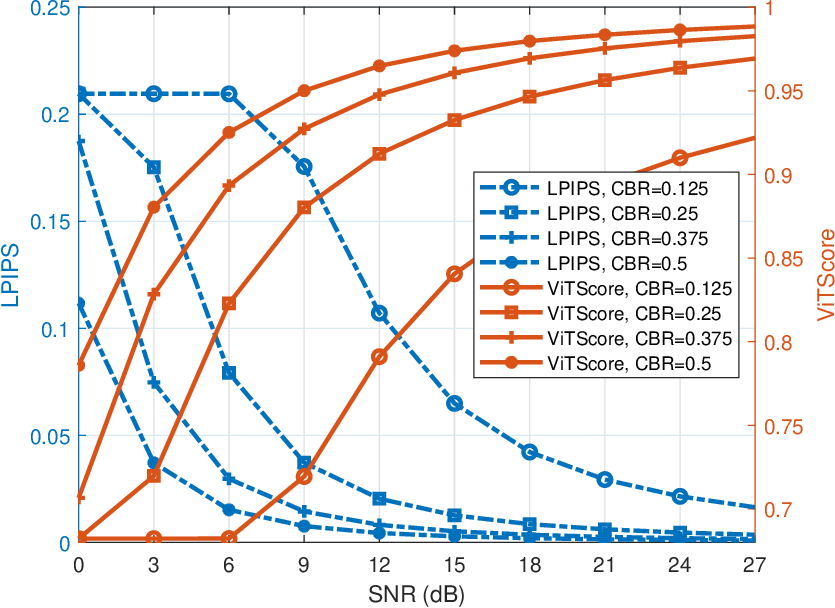}
  }
  \caption{The metric evaluation of the average performance of image transmission with the concatenation of JPEG code followed by an assuming capacity-achieving channel coding over an AWGN channel through COCO image dataset.
  Average reconstruction quality increases gradually with the channel bandwidth ratio ${\rm CBR}$ increasing, as well as the channel environment improving.
  The performance trends of \ours are in line with those of the other 3 typical metrics~(\psnr, \ms, and \lp).}\label{AWGN-ViT-JPEGcapacity}
\end{figure*}

\begin{figure*}[htbp]
  \centering
  \subfloat[\psnr vs \ours.]{
  \includegraphics[width=0.3\textwidth]{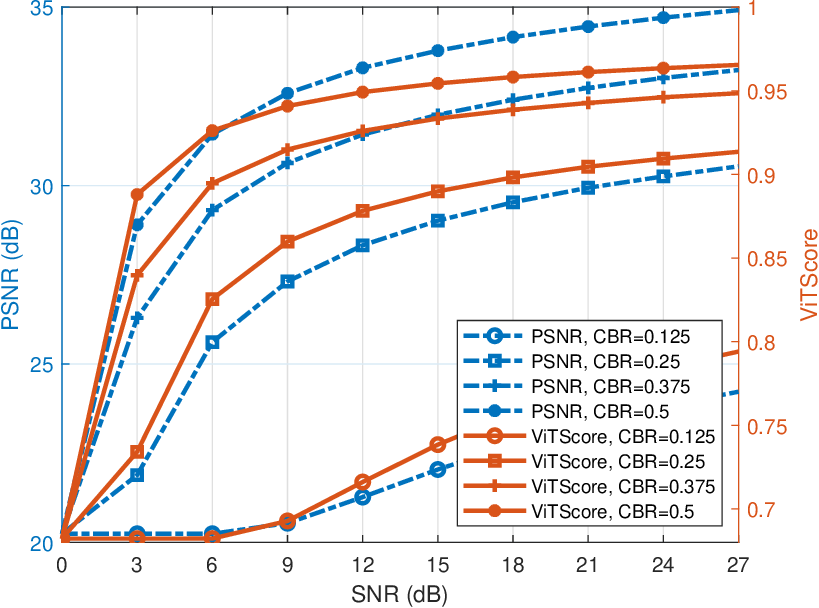}
  }
  \quad
  \subfloat[\ms vs \ours.]{
  \includegraphics[width=0.3\textwidth]{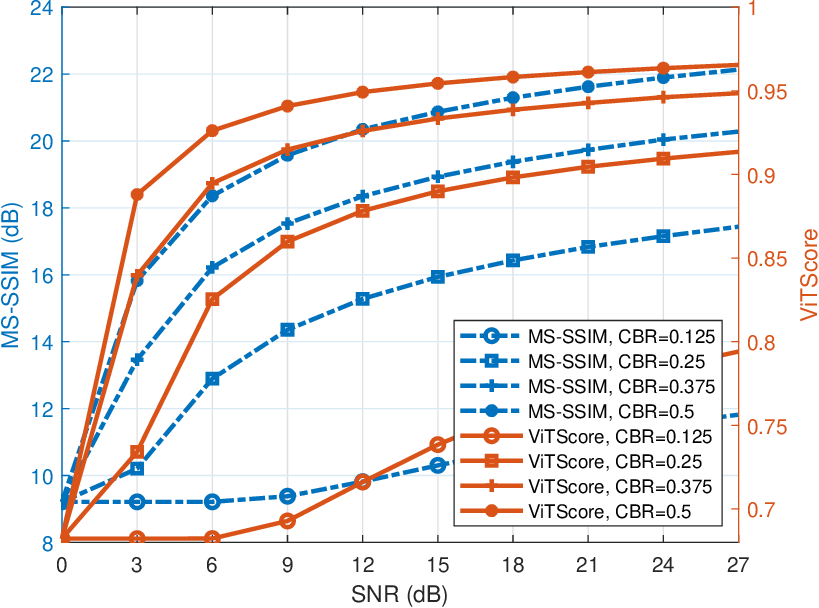}
  }
  \quad
  \subfloat[\lp vs \ours.]{
  \includegraphics[width=0.3\textwidth]{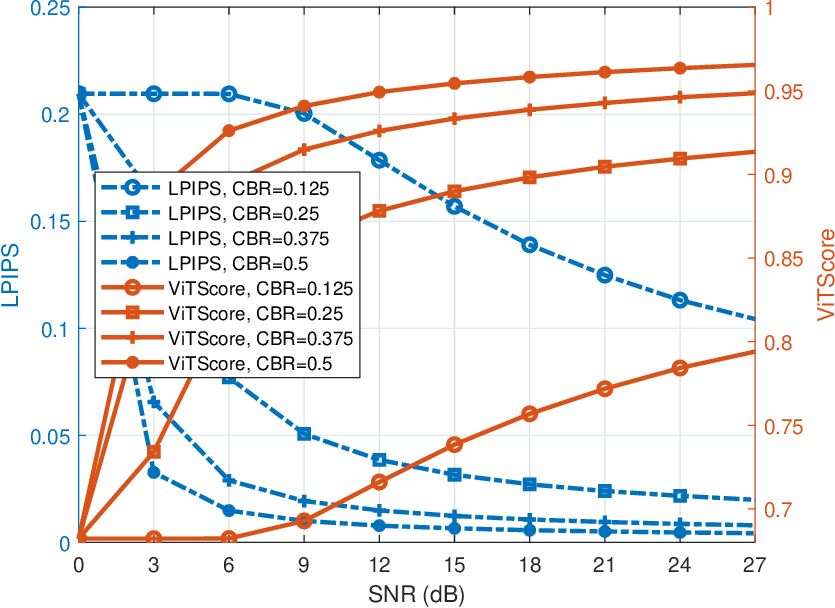}
  }
  \caption{The metric evaluation of the average performance of image transmission with the concatenation of JPEG code followed by an assuming capacity-achieving channel coding over a Rayleigh fading channel through COCO image dataset.
  Average reconstruction quality increases gradually with the channel bandwidth ratio ${\rm CBR}$ increasing, as well as the channel environment improving.
  The performance trends of \ours are in line with those of the other 3 typical metrics~(\psnr, \ms, and \lp).}\label{Fading-ViT-JPEGcapacity}
\end{figure*}

\subsection{Applied to Image Semantic Communication Systems}

In this subsection, we further analyze the application of \ours for image SC.
Since we have shown the \ours performance in SSCC scheme in the last subsection, here we test the \ours performance for image SC by employing a joint source-channel coding~(JSCC) scheme.
Similar to the landmark works of~\cite{dai2022nonlinear,bourtsoulatze2019deep,kurka2020deep,kurka2021bandwidth}, we utilized the proposed \ours metric to practical image SC over analog channels. To be specific, an image is first transformed into a vector of pixel intensities $\mathbf{x} \in \mathbb{R}^n$, then mapped to a vector of continuous-valued channel input symbols $\mathbf{s} \in \mathbb{R}^k$ via an AI-based encoding function. The channel bandwidth ratio is defined as ${\rm CBR} = k/n$, where $k < n$, denoting the code rate. Under this architecture, by introducing AI technology to extract the semantics of the image, the communication paradigm shifts to a modern version of JSCC.
\looseness=-1

We test the evaluation performance of \ours compared with the other 3 metrics~(\psnr, \ms and \lp) by utilizing 4 landmark image SC models: NTSCC~\cite{dai2022nonlinear}, DeepJSCC-V~\cite{zhang2023predictive}, ADJSCC~\cite{Jialong} and DeepJSCC~\cite{Bourtsoulatze}. All 4 models are based on DL techniques. Thus, the 4 models can be referred to as different variants of Deep JSCC, which is flexible and bandwidth efficient.
Strikingly, Deep JSCC does not suffer from the “cliff effect” presented in traditional image transmission systems, and it provides a graceful performance degradation as the channel SNR varies with respect to the SNR value assumed during training.

Deep JSCC for wireless image transmission~(named DeepJSCC) proposed in~\cite{Bourtsoulatze} is an end-to-end communication system, which does not rely on explicit codes for either compression or error correction, where the encoding and decoding functions are parameterized by two convolutional neural networks~(CNNs). The CNN can directly map the image pixel values to the complex-valued channel input symbols, and the communication channel is incorporated in the neural network~(NN) architecture as a non-trainable layer. That is where the name of DeepJSCC comes from.
Compared with the digital transmission concatenating JPEG or JPEG2000 compression with a capacity-achieving channel code at low SNR and CBR regions, DeepJSCC achieves superior performance.
As a new emerged methodology, the Deep JSCC models suffer from two problems: SNR-adaption problem and CBR-adaption problem.
To solve the SNR-adaption problem, an attention module-based Deep JSCC model~(named ADJSCC) is proposed in~\cite{Jialong}, which can adjust the learned image features in deferent channel SNR conditions.
Moreover, to solve both the SNR-adaption and the CBR-adaption problems, a novel image SC framework, named predictive and adaptive deep coding~(PADC), is proposed in~\cite{zhang2023predictive}. PADC is realized by a variable code length enabled Deep JSCC model~(named DeepJSCC-V) for realizing flexible code length adjustment, which is also named as ADJSCC-V. As demonstrated in~\cite{zhang2023predictive}, the DeepJSCC-V model can achieve similar PSNR performances compared with the ADJSCC model with flexible code length adjustment.
Furthermore, a new semantic encoding architecture with nonlinear transform, named nonlinear transform source-channel coding~(NTSCC) is proposed in~\cite{dai2022nonlinear}. 
The NTSCC yields end-to-end transmission performance surpassing classical separated-based BPG source compression combined with low-density parity-check~(LDPC) channel coding scheme and standard Deep JSCC methods.
Different from the DeepJSCC model, the NTSCC first learns a nonlinear analysis transform to map the source data into latent space, effectively extracting the source semantic features and providing side information for source-channel coding. This important step is defined as semantic encoding, which largely improves the overall SC system performance.

As our purpose is not to focus on the comparison with the performance of different models but to test the robustness of the proposed metric to the evaluation performance of different system models, we use different image datasets to train the 4 models, obtaining the pre-trained models (generated encoder-decoder pairs (EDP)), and then test on the same Kodak dataset~(one of the commonly used datasets in the 4 models presenting papers). Therefore, the performance of the 4 models may be slightly different from the presentation of the original papers (also caused by the different training techniques).
Experimental results are shown in Figs.~\ref{SC-SNR} and~\ref{SC-R0025}.
We can obtain that, similar to the SSCC scheme, the \ours and \lp evaluations also converge faster than the PSNR and MS-SSIM evaluations in the high SNR and CBR regions, where the semantics of the images are very similar.
Not only for SSCC scheme, but also the JSCC scheme, the performance threshold value may hard to improve at bit level even in high SNR and CBR regions.
However, the proposed metric evaluation is able to converge to a high level even in the low SNR and CBR regions, suggesting that the relatively low SNR or CBR is sufficient for the SC systems at the semantic level.

\begin{figure*}[htbp]
  \centering
  \subfloat[\ours vs \psnr.]{
  \includegraphics[width=0.3\textwidth]{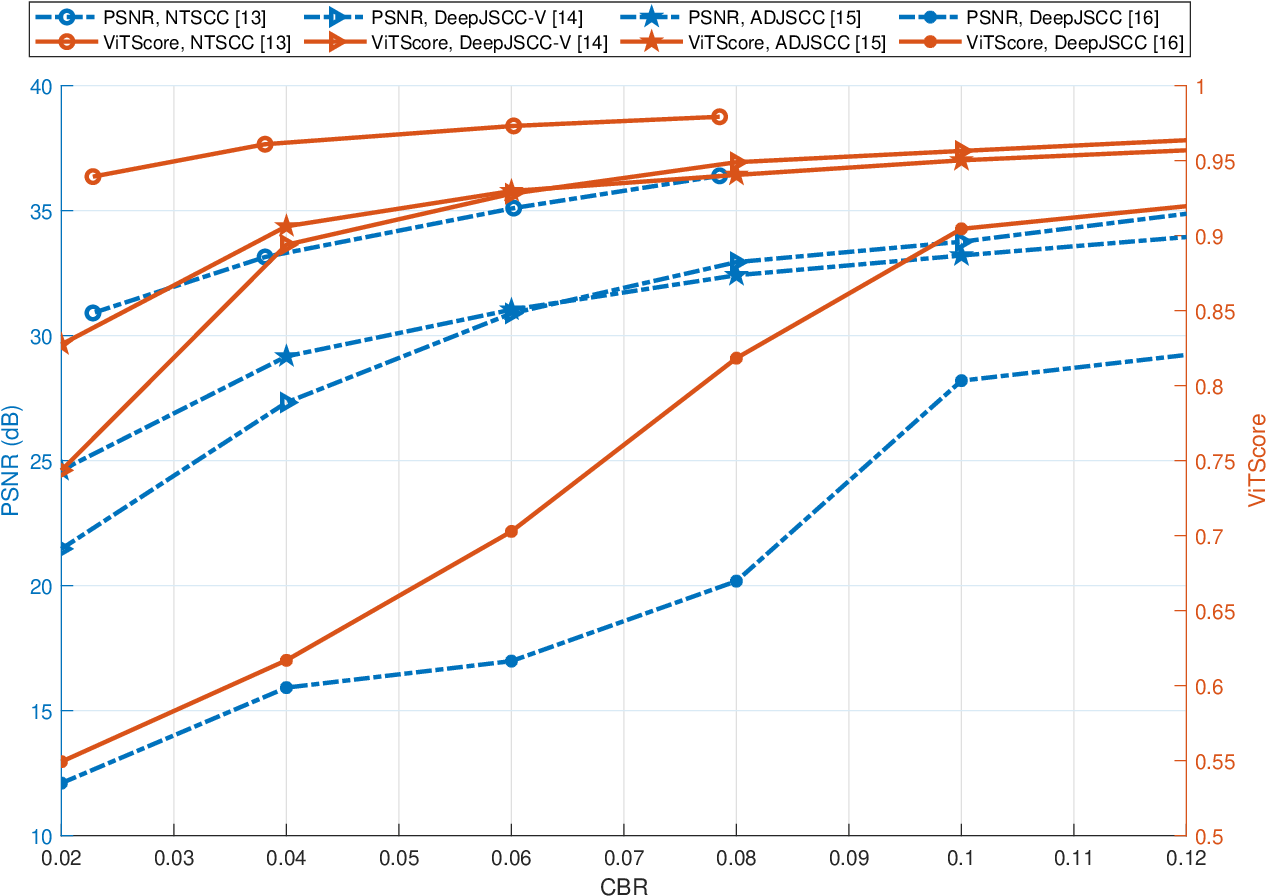}
  }
  \quad
  \subfloat[\ours vs \ms.]{
  \includegraphics[width=0.3\textwidth]{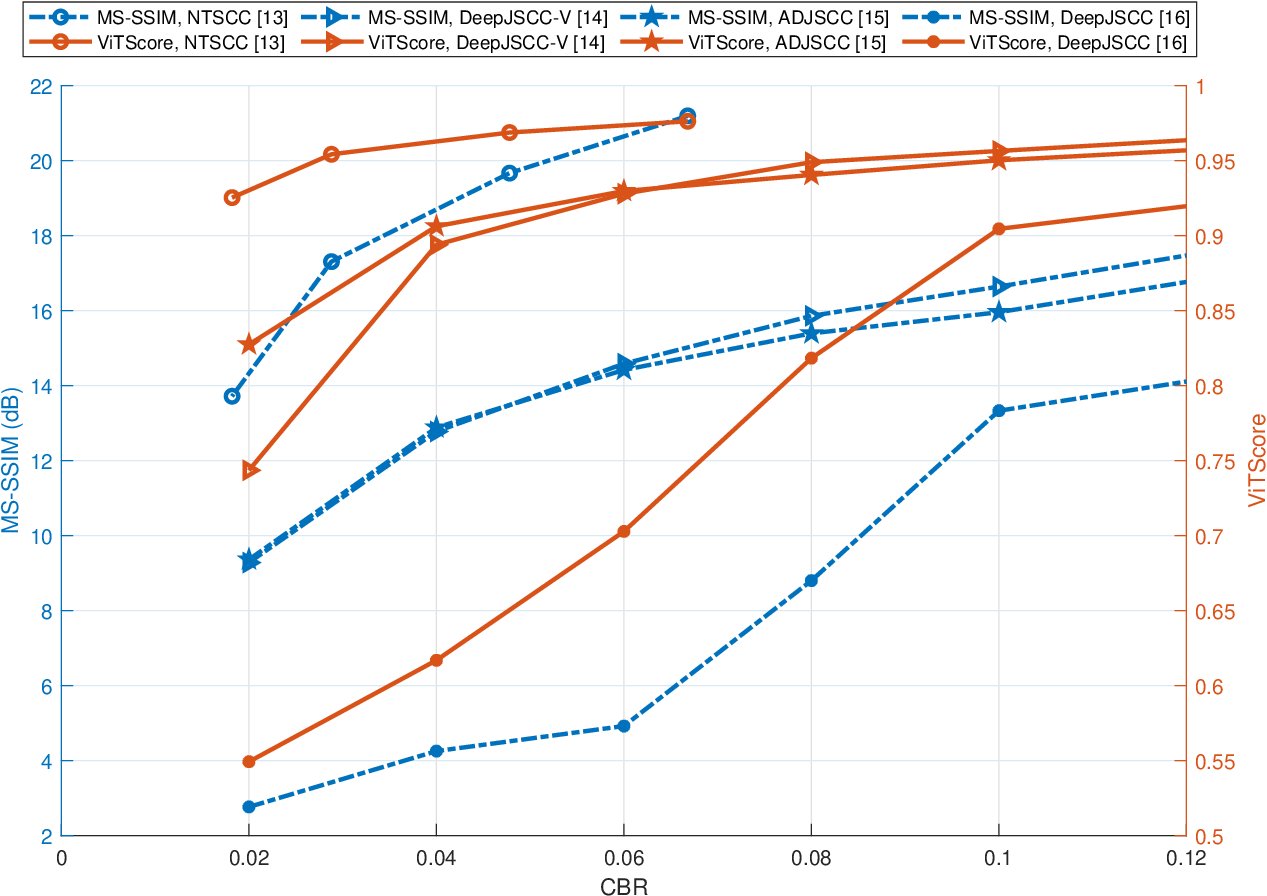}
  }
  \quad
  \subfloat[\ours vs \lp.]{
  \includegraphics[width=0.3\textwidth]{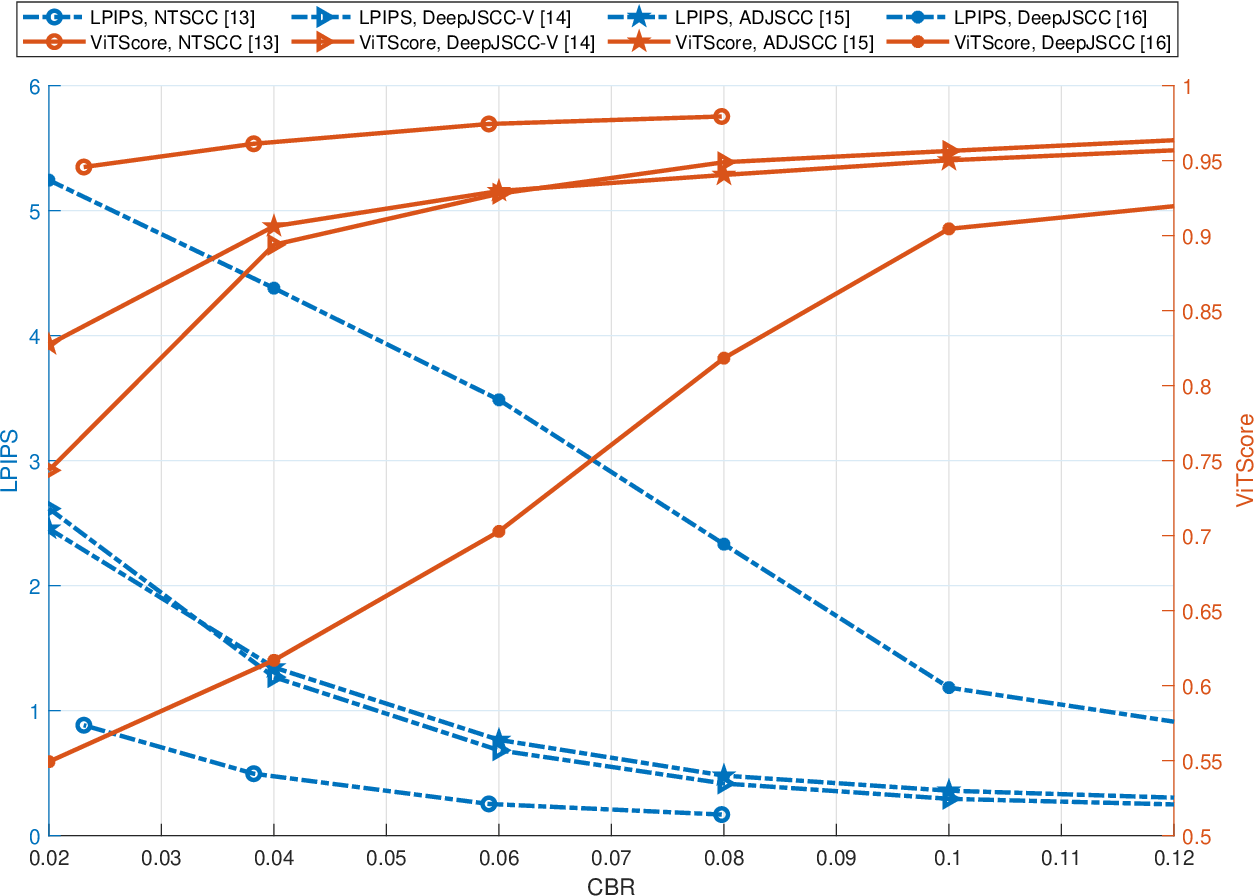}
  }
  \caption{The metric evaluation of the average performance of image semantic communications based on the 4 semantic communication models~(NTSCC~\cite{dai2022nonlinear}, DeepJSCC-V~\cite{zhang2023predictive}, ADJSCC~\cite{Jialong} and DeepJSCC~\cite{Bourtsoulatze}) over an AWGN channel at SNR = 15dB through the Kodak dataset. Average reconstruction quality increases gradually with the channel bandwidth ratio ${\rm CBR}$ increasing, as well as the channel environment improving.
  The performance trends of \ours are in line with those of the other 3 typical metrics~(\psnr, \ms, and \lp). }
  \label{SC-SNR}
\end{figure*}

\begin{figure*}[htbp]
  \centering
  \subfloat[\ours vs \psnr.]{
  \includegraphics[width=0.3\textwidth]{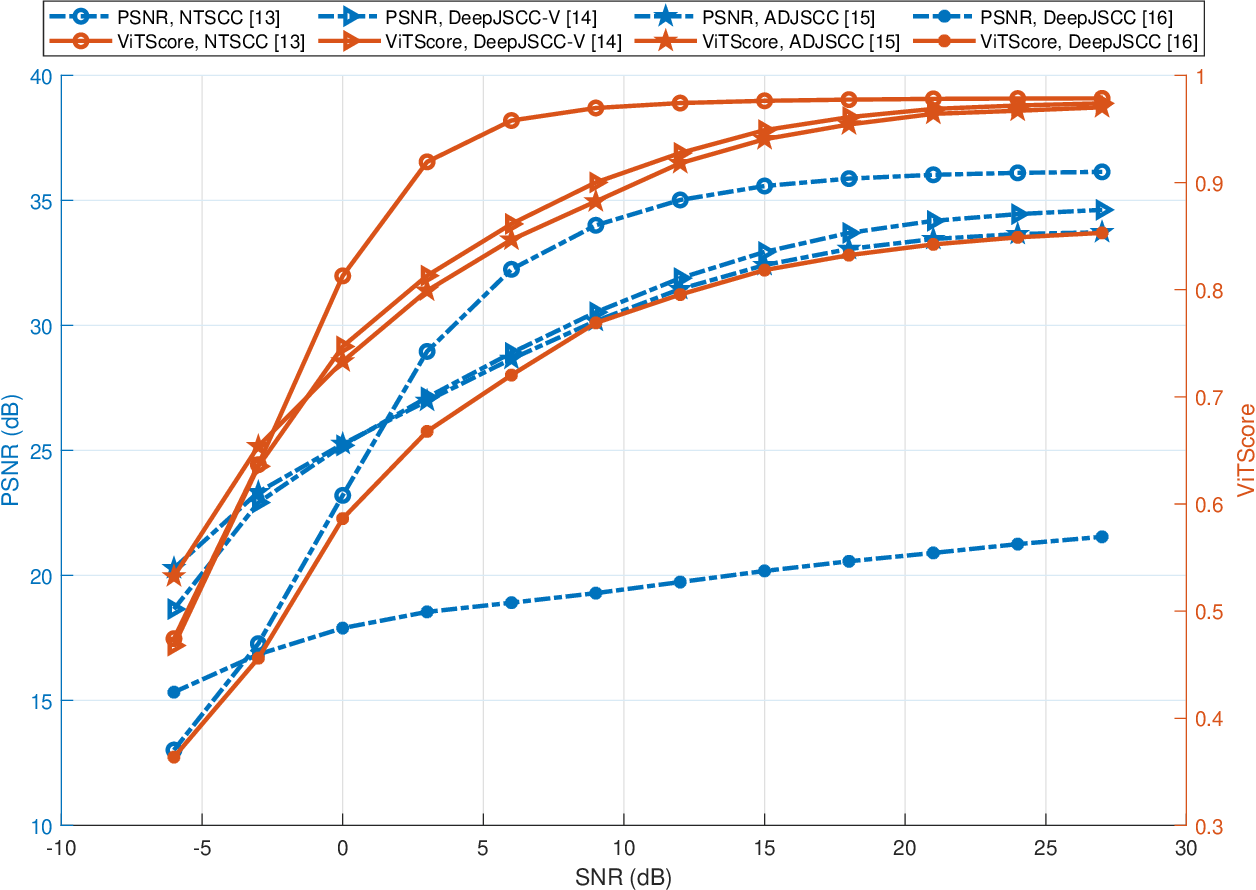}
  }
  \quad
  \subfloat[\ours vs \ms.]{
  \includegraphics[width=0.3\textwidth]{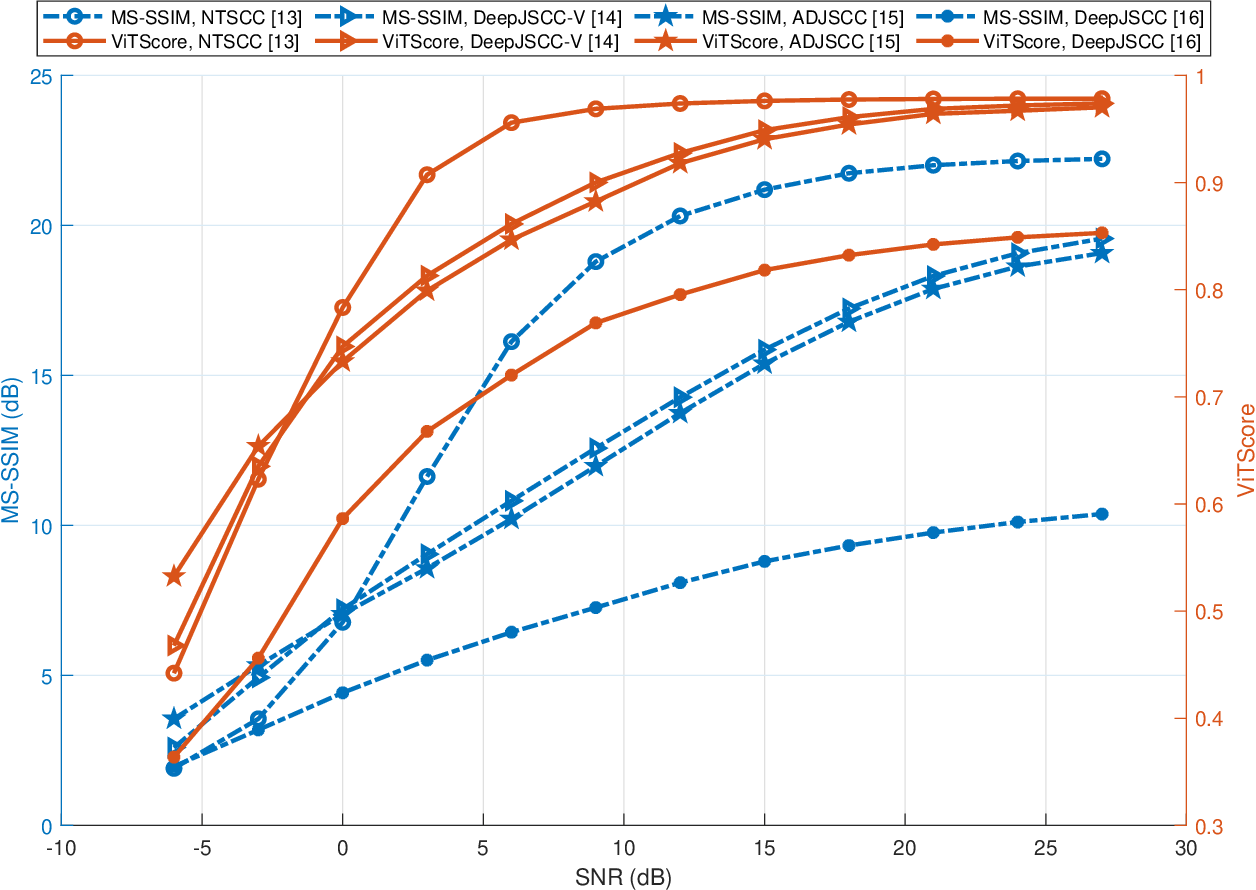}
  }
  \quad
  \subfloat[\ours vs \lp.]{
  \includegraphics[width=0.3\textwidth]{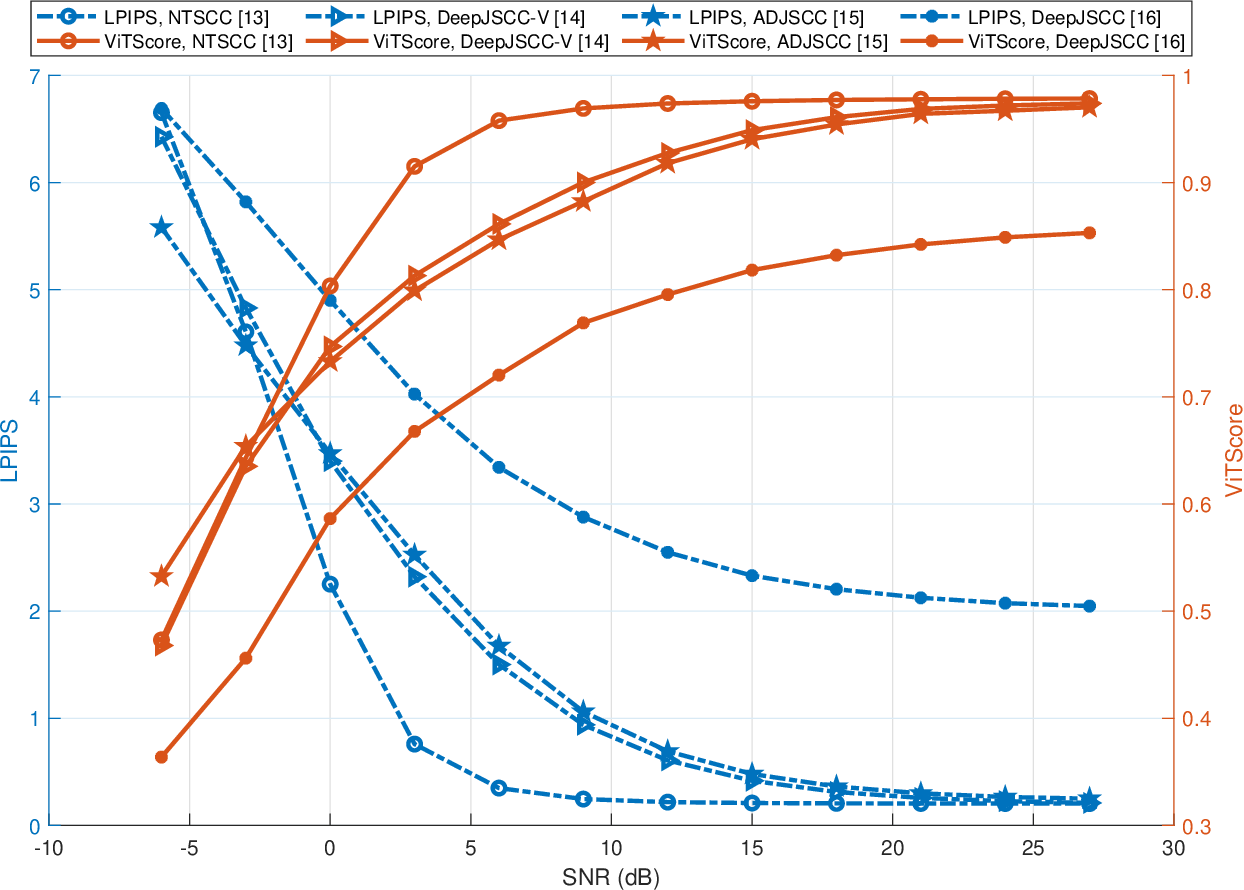}
  }
  \caption{The metric evaluation of the average performance of image semantic communications based on the 4 semantic communication models~(NTSCC~\cite{dai2022nonlinear}, DeepJSCC-V~\cite{zhang2023predictive}, ADJSCC~\cite{Jialong} and DeepJSCC~\cite{Bourtsoulatze}) versus the SNR at the AWGN channel through the Kodak dataset with the average channel bandwidth ratio fixed ${\rm CBR} \approx 0.08$.
  Average reconstruction quality increases gradually with the channel environment improving, i.e., the SNR increasing.
  The performance trends of \ours are in line with those of the other 3 typical metrics~(\psnr, \ms, and \lp).}
  \label{SC-R0025}
\end{figure*}

\subsection{Applied to Image Semantic Communication Systems with Semantic Attack}

Significantly different from classical communication systems, SC systems have semantic channels, where semantic noise exists. 
Semantic noise has a small or large impact on the SC system performance. For example, the semantic attack on the task-irrelevant semantics may have little effect on task execution. While the images are incorrectly semantic-encoded or the training datasets are semantically attacked, the performance of the AI model trained by the images may be affected seriously~\cite{du2023rethinking}.
Considering the evaluation of semantic changes is crucial for robust SC system design.
Specifically, for image SC scenarios, there exist two types of channels, including physical channels and semantic channels. How to evaluate and combat the physical channel impairment in wireless communications has been investigated a lot in the past. However, for semantic channels and semantic noise, it is a new problem.
Semantic noise could be introduced in semantic encoding, data transportation and semantic decoding processes of the SC systems~\cite{shi2021semantic}. Particularly, during the data transmission stage, the semantic noise can be generated by malicious attackers~\cite{hu2023robust}.
For image, the semantic noise can be introduced by applying the adversarial samples~\cite{szegedy2013intriguing}, which makes it possible to fail the networks' tasks.
In this subsection, we discuss how \ours performs in image SC with semantic noise.
Generally speaking, the existing image SC schemes can be divided into two categories: task-oriented and image reconstruction.
Therefore, we evaluate the metrics performance of image semantic changes with some existing typical works on the two categories of semantic attack as follows.
\looseness=-1

\textit{(1) Semantic Attack to Mislead Image Classification Task}

In~\cite{hu2023robust}, the authors proposed a framework for robust end-to-end SC systems to combat semantic noise, where the adversarial training with weight perturbation is developed to incorporate the samples with semantic noise in the training dataset. 
A masked vector quantized-variational autoencoder~(VQ-VAE) is utilized with the noise-related masking strategy.
With the masked VQ-VAE enabled codebook, which is shared by the transmitter and the receiver for encoded feature representation, the proposed model~(named Robust-SC) significantly improves the system robustness against semantic noise with a remarkable reduction in the transmission overhead.
Therefore, we utilize the pre-trained Robust-SC in~\cite{hu2023robust} as an image SC system model with semantic noise to compare the evaluation performance of \ours with the 3 typical metrics. 
We test the image semantic similarities on average over the first 10 categories of images from the large-scale 1000-class ImageNet-1K dataset~\cite{deng2009imagenet}.
\looseness=-1

Adversarial samples can fool the DL-based models and mislead the classification, however, the reconstructed images look identical to the original ones for humans, since the small perturbations caused by semantic noise added to the images are barely noticeable to human perception~\cite{Qin2022SemanticCP,hu2023robust}.
Fig.~\ref{RobustSemCom} shows a case evaluated by \ours and the 3 typical metrics of reconstruction performance with the Robust-SC model.
We can see that, the metrics give a good reconstruction quality assessment except \psnr.
Furthermore, we calculate the PCC between the metric and the system model cross-entropy~(CE) loss. 
The more the CE loss, the less the semantic similar of the images.
See Table.~\ref{tab: Robust-SC} for reference. We can find that, the \ours, \psnr and \ms are negatively correlated with the CE loss, while the \lp is positively correlated. 
The absolute value of PCC between \ours and the CE loss is the largest one, indicating that \ours can better reflect the change of the image semantics.
\looseness=-1

\begin{figure}[htbp]
  \centering
  \subfloat[Original image.]{
  \includegraphics[width=0.2\textwidth]{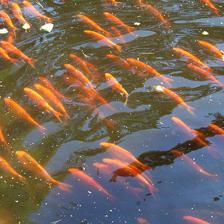}
  }
  \quad
  \subfloat[Reconstructed image with the Robust-SC model.]{
  \includegraphics[width=0.2\textwidth]{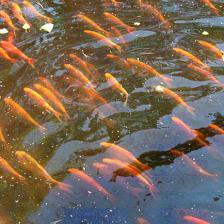}
  }
  \caption{Utilizing the Robust-SC model, the performance evaluation between an image (a) and its reconstruction (b): \psnr = 24.83, \ms = 0.96, \lp = 0.05, \ours = 0.91. 
  Generally, these metrics give a good reconstruction quality assessment except \psnr. }\label{RobustSemCom}
\end{figure}

\begin{table}[htbp]
    \renewcommand{\arraystretch}{1.5}
    \centering
    \caption{\label{tab: Robust-SC} Reconstruction performance evaluation with 4 metrics and their PCC with the system model CE loss in average over 10 categories images from ImageNet 1K dataset under accuracy performance of image classification with 60$\%$.}
    \begin{tabular}{|c|c|c|}
        \hline
         Metrics & Evaluation & PCC with cross-entropy \\
         \hline \hline
         $\ours$ & 0.9241  & -0.1625 \\
         \hline
         $\psnr$ & 26.9866  & -0.1065 \\
         \hline
         $\ms$ & 0.9655  & -0.1594 \\
         \hline
         $\lp$ & 0.0528  & 0.1384 \\
         \hline
    \end{tabular}
\end{table}

\textit{(2) Semantic Attack with Image Inverse}

Utilizing an in-domain GAN inversion approach~\cite{zhu2020indomain}, an image can be inverted to a latent code with a domain-guided encoder, which not only can be faithfully reconstructed but also can be semantically meaningful for editing.
As an emerging technique to bridge the real and fake
image domains, GAN inversion aims to invert a given image back into the latent space of a pre-trained GAN model so that the image
can be faithfully reconstructed from the inverted code by the generator~\cite{xia2022gan}.
The inverted codes
have important properties: having supported resolution,
being semantic-aware, being layer-wise, and having out-of-distribution generalizability, which make the GAN inversion methods have a widespread application in the image processing downstream tasks.
In~\cite{han2023generative}, the authors proposed a generative model-based image SC model to improve the efficiency of image transmission. With this model, the transmitter extracts the interpretable latent representation from the original image by a generative model exploiting the GAN inversion method. 
Motivated by this work, we utilized a similar method to simulate the image SC with semantic noise.
Specifically, we use the pre-trained In-Domain GAN Inversion model~\cite{zhu2020indomain} to build up an image SC model~(named ID-GAN-inv-based SC), where the physical channel is supposed to be noiseless for simplicity.

We test the image semantic similarities on average over 1000 face images in CelebA dataset~\cite{CelebA} and tower category validation set in LSUN dataset~\cite{yu2015lsun}.
The average simulation results are shown in Table.~\ref{tab: GAN-Inverse}.
Intuitively, from a typical case shown in Fig.~\ref{ID-GAN}, we can obtain that, transmitted with the ID-GAN-inv-based SC with/without image inverse attack, the evaluation of the image similarities may degrade a lot under the measurement of the \psnr, \ms and \lp. However, \ours shows that it keeps its semantic similarity to a large extent, which is consistent with visual perception.
Directly utilizing the pre-trained In-Domain GAN inversion model degrades the image quality, since it is not a tailored image SC model. For fairness, we compare the metric performance with the gap of the evaluation between with and without image inverse attack.
We can see that, the degradation of \ours is smaller than the other 3 metrics, indicating that \ours can better reflect the image semantic changes.
More strikingly, \ours outperforms the other 3 metrics in the cases where the distribution space of the reconstructed images is similar to that of the original ones, such as reconstruction with the image inverse transform.

\begin{figure*}[htbp]
  \centering
  \subfloat[Original image.]{
  \includegraphics[width=0.25\textwidth]{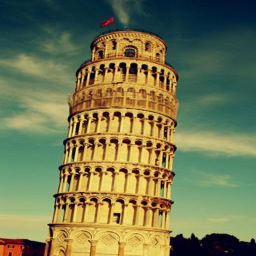}
  }
  \quad
  \subfloat[Reconstructed image with GANs-based SC system model.]{
  \includegraphics[width=0.25\textwidth]{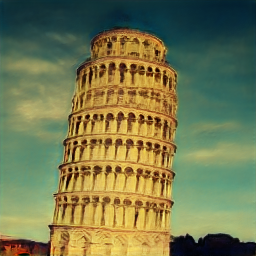}
  }
  \quad
  \subfloat[Reconstructed image with GANs-based SC system model under semantic noise interference.]{
  \includegraphics[width=0.25\textwidth]{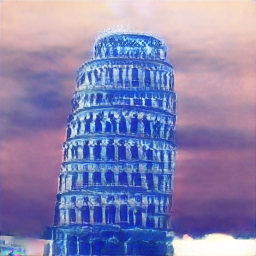}
  }
  \caption{Comparison of the image semantic similarity evaluation with 4 metrics.
  Transmitting an image (a) with the ID-GAN-inv-based SC model, the reconstruction performance (b) is evaluated: \psnr = 15.19, \ms = 0.69, \lp = 0.09, \ours = 0.87.
  While the semantic noise is introduced, which is characterized as the image inverse with GANs, the reconstruction performance (c) is evaluated: \psnr = 3.12, \ms = 0.00, \lp = 0.38, \ours= 0.73.
  The fluctuation in terms of the other 3 metrics is much larger than that of \ours. However, the semantic similarities between (a) and (b) (resp. (a) and (c)) are very close.}
  \label{ID-GAN}
\end{figure*}

\begin{table*}[htb]
    \renewcommand{\arraystretch}{1.5}
    \centering
    \caption{\label{tab: GAN-Inverse} Reconstruction performance evaluation with 4 metrics in average over 2 datasets for GANs-based SC model with and without semantic noise, which is characterized as the image inverse by generating with GANs.}
  \begin{tabular}{|c|c|c|c|c|}
  \hline
  \multirow{2}{*}{Metrics} & \multicolumn{2}{c|}{1K face images in CelebA dataset}                                                                 & \multicolumn{2}{c|}{Tower category validation set in LSUN dataset}                                                                \\ \cline{2-5} 
  & \multicolumn{1}{c|}{ID-GAN-inv-based SC} & \multicolumn{1}{c|}{\begin{tabular}[c]{@{}c@{}}ID-GAN-inv-based SC \\ with image inverse attack\end{tabular}} & \multicolumn{1}{c|}{ID-GAN-inv-based SC} & \multicolumn{1}{c|}{\begin{tabular}[c]{@{}c@{}}ID-GAN-inv-based SC \\ with image inverse attack\end{tabular}} \\ \hline
  \multicolumn{1}{|c|}{\ours}  & \multicolumn{1}{c|}{0.7744}  & \multicolumn{1}{c|}{0.6190} & \multicolumn{1}{c|}{0.7237}  & \multicolumn{1}{c|}{0.5586}\\ 
  \hline
  \multicolumn{1}{|c|}{\psnr}  & \multicolumn{1}{c|}{18.4231}  & \multicolumn{1}{c|}{0.9665} & \multicolumn{1}{c|}{15.6117} & \multicolumn{1}{c|}{1.6179} \\ 
  \hline
  \multicolumn{1}{|c|}{\ms}  & \multicolumn{1}{c|}{0.8200}  & \multicolumn{1}{c|}{0.0012} & \multicolumn{1}{c|}{0.6832} & \multicolumn{1}{c|}{0.0017} \\ 
  \hline
  \multicolumn{1}{|c|}{\lp}  & \multicolumn{1}{c|}{0.1108}  & \multicolumn{1}{c|}{0.4580} & \multicolumn{1}{c|}{0.1966} & \multicolumn{1}{c|}{0.5008} \\ 
  \hline
\end{tabular}
\end{table*}

\textit{(3) Semantic Attack with Typical Image Transforms}

Utilizing the GAN inversion method, a variance of image semantic attack can be implemented in the GANs-based SC model.
For simplicity, we use the image-processing cases instead of training more GANs models to simulate the image SC with the semantic attack.
We take for example 7 typical image-processing cases to show that the proposed \ours evaluates semantic similarities better than 3 typical categories of metrics~(\psnr, \ms, and \lp), consisting of an image versus its inverse (obtained by 1 minus all pixels of the original image), its gray-scale version (obtained using Grayscale-operation by torch-vision), its flipping (vertical flipping and horizontal flipping), its rotation (180-degree rotation and 90-degree rotation) and versus random noise (uniformly generated in [0,1]).
The overall results are shown in Fig. \ref{plt:metrics}. We can find that with \psnr and \ms evaluation, the scores in most cases are similar to the random noise case, except the gray-scale case. However, with \ours and \lp evaluation, the scores in all other cases are higher than the random noise case by large margins.
The results show that \ours has excellent performance for image semantic similarity compared with the other 3 typical metrics (\psnr, \ms, and \lp). Processed by basic transforms, an image may keep its semantic information by a large margin. \ours between the original and transformed images may stay at a high level, while the \psnr, \ms, and \lp between them state a considerably low level. Recall that the higher the \lp score, the lower the similarity of images. As the \ours measures the image similarity at the global semantic level, instead of the pixel or structural or local perceptual level.
\looseness=-1

\begin{figure}[htbp]
  \centering
  \subfloat[\psnr.]{
  \includegraphics[width=0.21\textwidth]{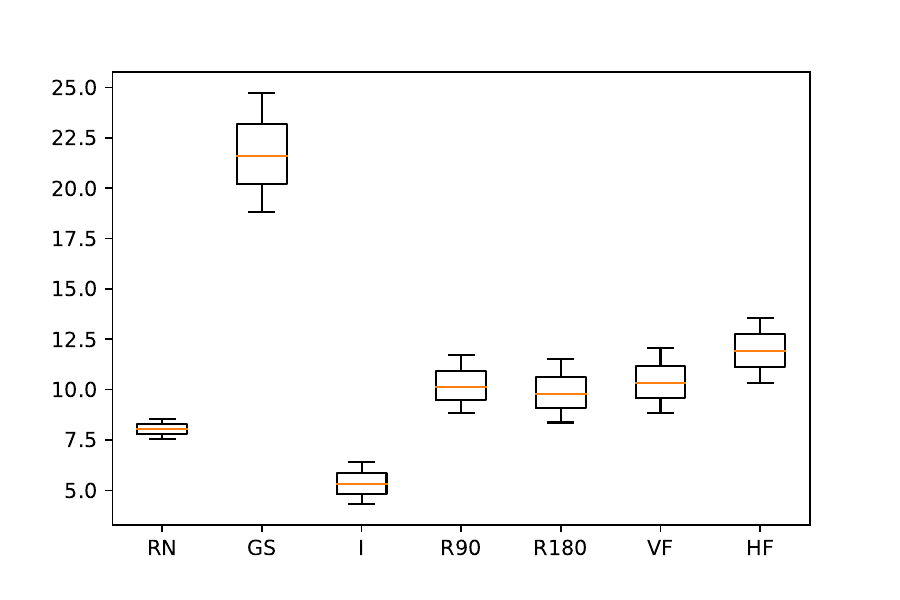}
  }
  \quad
  \subfloat[\ms.]{
  \includegraphics[width=0.21\textwidth]{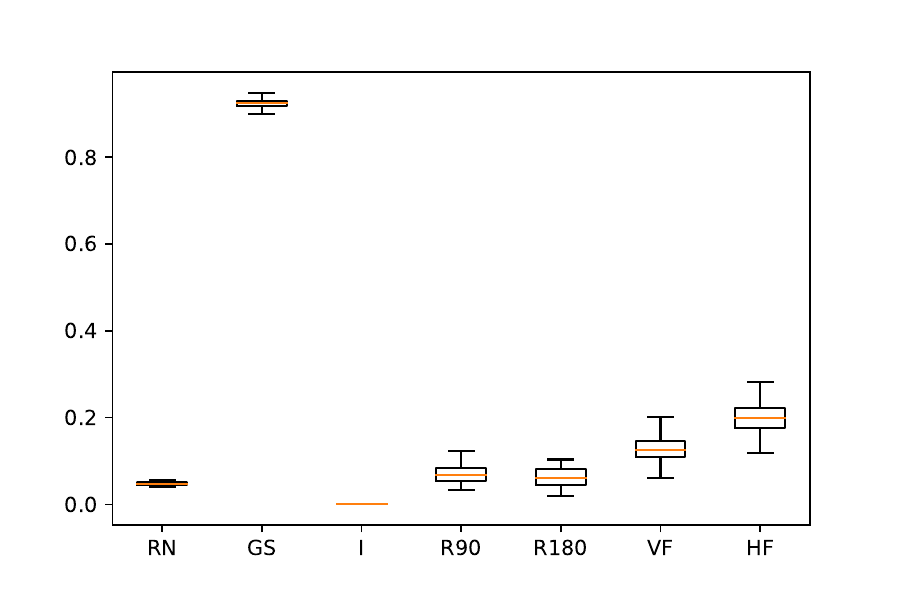}
  }
  \quad
  \subfloat[\lp.]{
  \includegraphics[width=0.21\textwidth]{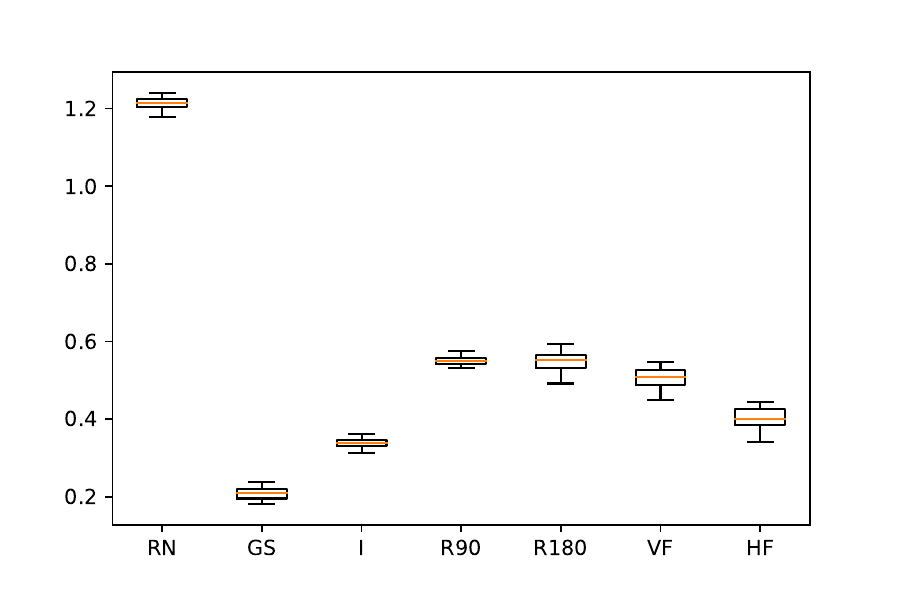}
  }
  \quad
  \subfloat[\ours.]{
  \includegraphics[width=0.21\textwidth]{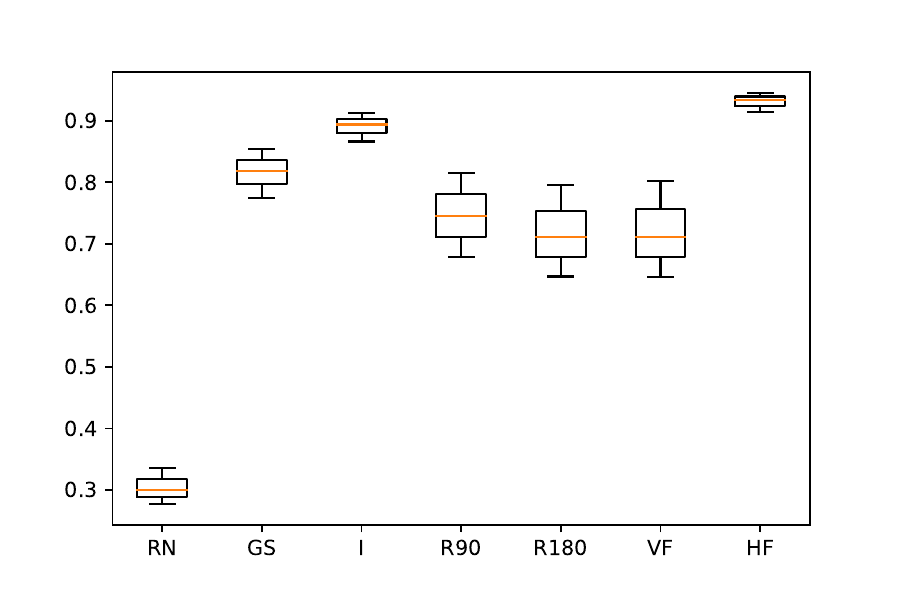}
  }
  \caption{The average \psnr, \ms, \lp and \ours evaluation of an image versus random noise (RN), its gray-scale version (GS), its inverse (I), its 90-degree rotation (R90), its 180-degree rotation (R180), its vertical flipping (VF), and its horizontal flipping (HF) over DIV2K dataset.
  Except for the gray-scale case, the \psnr and \ms evaluations in most cases are similar to the random noise case.
  Whereas, the \lp and \ours evaluations in all other cases are lower than the random noise case by large margins.}
  \label{plt:metrics}
\end{figure}

Furthermore, we study the universality of the \ours metric over 6 typical image datasets (DIV2K, Set5, Set14, B100, Urban100 and COCO).
Table~\ref{universality} shows the experimental results in terms of \ours and 3 typical categories of metrics (\psnr, \ms, and \lp). In this table, $\overline{r}$ is the mean of raw scores, and $\overline{s}$ is the mean of standard scores. Standard scores are calculated by the following equation:
\begin{equation}
    \label{eq:standard_score}
    s_{(\mathrm{Dataset}, \mathrm{Metric})}(r) = \mathrm{sign}_\mathrm{Metric}\frac{r - \mu_{(\mathrm{Dataset}, \mathrm{Metric})}}{\sigma_{(\mathrm{Dataset}, \mathrm{Metric})}},
\end{equation}

\noindent where $\mu_{(\mathrm{Dataset}, \mathrm{Metric})}$ and $\sigma_{(\mathrm{Dataset}, \mathrm{Metric})}$ are mean value and standard deviation of evaluation scores of any two images in the given dataset with the given metric. $\mathrm{sign}_\mathrm{Metric}=-1$ if the metric is LPIPS (since it gives smaller scores for similar images), otherwise $\mathrm{sign}_\mathrm{Metric}=1$.
We can observe that the \ours metric is universal for image measurement over all the above-mentioned datasets. It reaches the best or the second-best performance in most experiments except the low-resolution experiment and the gray-scale experiment.

\begin{table*}[htbp]
  \renewcommand\arraystretch{1.5}
  \caption{\label{universality} The average \ours performance compared with PSNR, MS-SSIM, and LPIPS over the image datasets DIV2K, Set5, Set14, B100, Urban100 and COCO with 6 experiments. The best (lowest in random noise and highest in others) results are in \textbf{bold}, and the second-best results are with \underline{underline}. Definitions of $\overline{r}$ and $\overline{s}$ are shown in Equation (\ref{eq:standard_score}).}
  \setlength{\tabcolsep}{4pt}{
    \centering  
    \scalebox{1.2}{
  \begin{tabular}{|c|c|cccccccccccc|}
  \hline
  \multirow{3}{*}{Experiment} & \multirow{3}{*}{Metric} & \multicolumn{12}{c|}{Dataset}                                                                                                 \\ \cline{3-14} 
                              &                         & \multicolumn{2}{c|}{DIV2K} & \multicolumn{2}{c|}{Set5}  & \multicolumn{2}{c|}{Set14} & \multicolumn{2}{c|}{B100} & \multicolumn{2}{c|}{Urban100} & \multicolumn{2}{c|}{COCO} \\
                              \cline{3-14} 
                              &                         & \multicolumn{1}{c|}{$\overline{r}$} & \multicolumn{1}{c|}{$\overline{s}$} & \multicolumn{1}{c|}{$\overline{r}$} & \multicolumn{1}{c|}{$\overline{s}$} & \multicolumn{1}{c|}{$\overline{r}$} & \multicolumn{1}{c|}{$\overline{s}$} & \multicolumn{1}{c|}{$\overline{r}$} & \multicolumn{1}{c|}{$\overline{s}$} & \multicolumn{1}{c|}{$\overline{r}$} & \multicolumn{1}{c|}{$\overline{s}$} & \multicolumn{1}{c|}{$\overline{r}$} & \multicolumn{1}{c|}{$\overline{s}$} \\
                              \hline \hline 
                              \multirow{2}{*}{\shortstack{An image\\versus its\\inverse}}          & \ours                       & \multicolumn{1}{c|}{0.88} & \multicolumn{1}{c|}{\textbf{6.42}}  & \multicolumn{1}{c|}{0.92} & \multicolumn{1}{c|}{\textbf{6.61}}  & \multicolumn{1}{c|}{0.86} & \multicolumn{1}{c|}{\textbf{7.93}}  & \multicolumn{1}{c|}{0.89} & \multicolumn{1}{c|}{\textbf{5.61}}  & \multicolumn{1}{c|}{0.87} & \multicolumn{1}{c|}{\underline{4.83}} & \multicolumn{1}{c|}{0.88} & \multicolumn{1}{c|}{\textbf{7.15}} \\ \cline{2-14} 
                              & PSNR(dB)                       & \multicolumn{1}{c|}{5.14} & \multicolumn{1}{c|}{-2.37}  & \multicolumn{1}{c|}{4.22} & \multicolumn{1}{c|}{-4.13}  & \multicolumn{1}{c|}{5.45} & \multicolumn{1}{c|}{-2.09}  & \multicolumn{1}{c|}{6.01} & \multicolumn{1}{c|}{-1.88} & \multicolumn{1}{c|}{5.42} & \multicolumn{1}{c|}{-2.70}  & \multicolumn{1}{c|}{5.09} & \multicolumn{1}{c|}{-2.21}   \\ \cline{2-14} 
  & MS-SSIM                       & \multicolumn{1}{c|}{0.00} & \multicolumn{1}{c|}{-0.97} & \multicolumn{1}{c|}{0.00} & \multicolumn{1}{c|}{-1.42} & \multicolumn{1}{c|}{0.00} & \multicolumn{1}{c|}{-0.99} & \multicolumn{1}{c|}{0.00} & \multicolumn{1}{c|}{-1.03} & \multicolumn{1}{c|}{0.00} & \multicolumn{1}{c|}{-1.00} & 
  \multicolumn{1}{c|}{0.00} & \multicolumn{1}{c|}{-0.91} \\ \cline{2-14} 
                              & LPIPS                       & \multicolumn{1}{c|}{0.34} & \multicolumn{1}{c|}{\underline{5.04}} & \multicolumn{1}{c|}{0.40} & \multicolumn{1}{c|}{\underline{4.63}} & \multicolumn{1}{c|}{0.35} & \multicolumn{1}{c|}{\underline{5.74}} & \multicolumn{1}{c|}{0.33} & \multicolumn{1}{c|}{\underline{5.02}} & \multicolumn{1}{c|}{0.25} & \multicolumn{1}{c|}{\textbf{6.45}} & 
                              \multicolumn{1}{c|}{0.34} & \multicolumn{1}{c|}{\underline{6.13}}  \\ \hline \hline
  \multirow{2}{*}{\shortstack{An image\\versus its\\90-degree\\rotation}}          & \ours                       & \multicolumn{1}{c|}{0.74} & \multicolumn{1}{c|}{\textbf{4.74}}  & \multicolumn{1}{c|}{0.82} & \multicolumn{1}{c|}{\textbf{5.51}}  & \multicolumn{1}{c|}{0.74} & \multicolumn{1}{c|}{\textbf{6.21}}  & \multicolumn{1}{c|}{0.75} & \multicolumn{1}{c|}{\textbf{4.04}}  & \multicolumn{1}{c|}{0.72} & \multicolumn{1}{c|}{\textbf{2.97}} & \multicolumn{1}{c|}{0.72} & \multicolumn{1}{c|}{\textbf{5.03}}     \\ \cline{2-14} 
                              & PSNR(dB)                       & \multicolumn{1}{c|}{9.93} & \multicolumn{1}{c|}{0.68}  & \multicolumn{1}{c|}{8.88} & \multicolumn{1}{c|}{0.79}  & \multicolumn{1}{c|}{9.99} & \multicolumn{1}{c|}{0.60}  & \multicolumn{1}{c|}{10.76} & \multicolumn{1}{c|}{0.32} & \multicolumn{1}{c|}{9.44} & \multicolumn{1}{c|}{0.44} & \multicolumn{1}{c|}{9.54} & \multicolumn{1}{c|}{0.54} \\ \cline{2-14} 
  & MS-SSIM                       & \multicolumn{1}{c|}{0.07} & \multicolumn{1}{c|}{0.27} & \multicolumn{1}{c|}{0.07} & \multicolumn{1}{c|}{0.88} & \multicolumn{1}{c|}{0.04} & \multicolumn{1}{c|}{-0.07} & \multicolumn{1}{c|}{0.08} & \multicolumn{1}{c|}{0.10} & \multicolumn{1}{c|}{0.05} & \multicolumn{1}{c|}{0.03} & 
  \multicolumn{1}{c|}{0.08} & \multicolumn{1}{c|}{0.26}     \\ \cline{2-14} 
                              & LPIPS                       & \multicolumn{1}{c|}{0.55} & \multicolumn{1}{c|}{\underline{1.61}} & \multicolumn{1}{c|}{0.51} & \multicolumn{1}{c|}{\underline{2.06}} & \multicolumn{1}{c|}{0.56} & \multicolumn{1}{c|}{\underline{2.10}} & \multicolumn{1}{c|}{0.54} & \multicolumn{1}{c|}{\underline{1.47}} & \multicolumn{1}{c|}{0.57} & \multicolumn{1}{c|}{\underline{1.44}} & 
                              \multicolumn{1}{c|}{0.57} & \multicolumn{1}{c|}{\underline{1.53}}  \\ \hline \hline
  \multirow{2}{*}{\shortstack{An image\\versus its\\180-degree\\rotation}}          & \ours                       & \multicolumn{1}{c|}{0.72} & \multicolumn{1}{c|}{\textbf{4.43}}  & \multicolumn{1}{c|}{0.78} & \multicolumn{1}{c|}{\textbf{5.01}}  & \multicolumn{1}{c|}{0.72} & \multicolumn{1}{c|}{\textbf{6.00}}  & \multicolumn{1}{c|}{0.72} & \multicolumn{1}{c|}{\textbf{3.66}} & \multicolumn{1}{c|}{0.72} & \multicolumn{1}{c|}{\textbf{2.93}} & \multicolumn{1}{c|}{0.67} & \multicolumn{1}{c|}{\textbf{4.46}}     \\ \cline{2-14} 
                              & PSNR(dB)                       & \multicolumn{1}{c|}{9.38} & \multicolumn{1}{c|}{0.33}  & \multicolumn{1}{c|}{9.39} & \multicolumn{1}{c|}{1.33}  & \multicolumn{1}{c|}{9.41} & \multicolumn{1}{c|}{0.26}  & \multicolumn{1}{c|}{10.24} & \multicolumn{1}{c|}{0.08}  & \multicolumn{1}{c|}{9.04} & \multicolumn{1}{c|}{0.13} & \multicolumn{1}{c|}{8.99} & \multicolumn{1}{c|}{0.20}     \\ \cline{2-14} 
  & MS-SSIM                       & \multicolumn{1}{c|}{0.06} & \multicolumn{1}{c|}{0.12} & \multicolumn{1}{c|}{0.06} & \multicolumn{1}{c|}{0.45} & \multicolumn{1}{c|}{0.03} & \multicolumn{1}{c|}{-0.21} & \multicolumn{1}{c|}{0.07} & \multicolumn{1}{c|}{-0.02} & \multicolumn{1}{c|}{0.05} & \multicolumn{1}{c|}{-0.02} & 
  \multicolumn{1}{c|}{0.06} & \multicolumn{1}{c|}{0.05}     \\ \cline{2-14} 
                              & LPIPS                       & \multicolumn{1}{c|}{0.55} & \multicolumn{1}{c|}{\underline{1.61}} & \multicolumn{1}{c|}{0.51} & \multicolumn{1}{c|}{\underline{2.19}} & \multicolumn{1}{c|}{0.56} & \multicolumn{1}{c|}{\underline{2.13}} & \multicolumn{1}{c|}{0.54} & \multicolumn{1}{c|}{\underline{1.46}} & \multicolumn{1}{c|}{0.55} & \multicolumn{1}{c|}{\underline{1.70}} & 
                              \multicolumn{1}{c|}{0.57} & \multicolumn{1}{c|}{\underline{1.51}}  \\ \hline \hline
   \multirow{2}{*}{\shortstack{An image\\versus\\random\\noise}}          & \ours                       & \multicolumn{1}{c|}{0.31} & \multicolumn{1}{c|}{\underline{-0.59}}  & \multicolumn{1}{c|}{0.32} & \multicolumn{1}{c|}{-0.36}  & \multicolumn{1}{c|}{0.29} & \multicolumn{1}{c|}{-0.01}  & \multicolumn{1}{c|}{0.31} & \multicolumn{1}{c|}{\underline{-0.92}} & \multicolumn{1}{c|}{0.27} & \multicolumn{1}{c|}{\underline{-2.62}} & \multicolumn{1}{c|}{0.30} & \multicolumn{1}{c|}{-0.36}     \\ \cline{2-14} 
                              & PSNR(dB)                       & \multicolumn{1}{c|}{7.96} & \multicolumn{1}{c|}{-0.58}  & \multicolumn{1}{c|}{7.49} & \multicolumn{1}{c|}{\underline{-0.66}}  & \multicolumn{1}{c|}{8.11} & \multicolumn{1}{c|}{\underline{-0.52}}  & \multicolumn{1}{c|}{8.34} & \multicolumn{1}{c|}{-0.80} & \multicolumn{1}{c|}{8.09} & \multicolumn{1}{c|}{-0.62} & \multicolumn{1}{c|}{7.93} & \multicolumn{1}{c|}{\underline{-0.46}}     \\ \cline{2-14} 
  & MS-SSIM                       & \multicolumn{1}{c|}{0.05} & \multicolumn{1}{c|}{-0.13} & \multicolumn{1}{c|}{0.04} & \multicolumn{1}{c|}{0.04} & \multicolumn{1}{c|}{0.04} & \multicolumn{1}{c|}{0.10} & \multicolumn{1}{c|}{0.06} & \multicolumn{1}{c|}{-0.21} & \multicolumn{1}{c|}{0.05} & \multicolumn{1}{c|}{0.05} & 
  \multicolumn{1}{c|}{0.05} & \multicolumn{1}{c|}{-0.16}     \\ \cline{2-14} 
                              & LPIPS                       & \multicolumn{1}{c|}{1.21} & \multicolumn{1}{c|}{\textbf{-9.43}} & \multicolumn{1}{c|}{1.25} & \multicolumn{1}{c|}{\textbf{-14.49}} & \multicolumn{1}{c|}{1.25} & \multicolumn{1}{c|}{\textbf{-9.35}} & \multicolumn{1}{c|}{1.19} & \multicolumn{1}{c|}{\textbf{-9.60}} & \multicolumn{1}{c|}{1.21} & \multicolumn{1}{c|}{\textbf{-8.66}} & 
                              \multicolumn{1}{c|}{1.24} & \multicolumn{1}{c|}{\textbf{-11.30}}  \\ \hline  \hline
  \multirow{2}{*}{\shortstack{An image\\versus its\\low-resolution\\version}}          & \ours                       & \multicolumn{1}{c|}{0.63} & \multicolumn{1}{c|}{3.30}  & \multicolumn{1}{c|}{0.77} & \multicolumn{1}{c|}{4.91}  & \multicolumn{1}{c|}{0.63} & \multicolumn{1}{c|}{4.69}  & \multicolumn{1}{c|}{0.63} & \multicolumn{1}{c|}{2.68} & \multicolumn{1}{c|}{0.59} & \multicolumn{1}{c|}{1.34} & \multicolumn{1}{c|}{0.61} & \multicolumn{1}{c|}{3.64}     \\ \cline{2-14} 
                              & PSNR(dB)                       & \multicolumn{1}{c|}{20.00} & \multicolumn{1}{c|}{\underline{7.09}} & \multicolumn{1}{c|}{20.34} & \multicolumn{1}{c|}{\underline{12.89}} & \multicolumn{1}{c|}{19.73} & \multicolumn{1}{c|}{\underline{6.38}} & \multicolumn{1}{c|}{21.61} & \multicolumn{1}{c|}{\underline{5.37}} & \multicolumn{1}{c|}{17.73} & \multicolumn{1}{c|}{\underline{6.90}} & \multicolumn{1}{c|}{20.74} & \multicolumn{1}{c|}{\underline{7.44}}    \\ \cline{2-14} 
  & MS-SSIM                       & \multicolumn{1}{c|}{0.75} & \multicolumn{1}{c|}{\textbf{12.17}} & \multicolumn{1}{c|}{0.80} & \multicolumn{1}{c|}{\textbf{25.61}} & \multicolumn{1}{c|}{0.76} & \multicolumn{1}{c|}{\textbf{17.57}} & \multicolumn{1}{c|}{0.77} & \multicolumn{1}{c|}{\textbf{10.16}} & \multicolumn{1}{c|}{0.67} & \multicolumn{1}{c|}{\textbf{13.13}} & 
  \multicolumn{1}{c|}{0.79} & \multicolumn{1}{c|}{\textbf{11.47}}     \\ \cline{2-14} 
                              & LPIPS                       & \multicolumn{1}{c|}{0.61} & \multicolumn{1}{c|}{0.60} & \multicolumn{1}{c|}{0.52} & \multicolumn{1}{c|}{1.88} & \multicolumn{1}{c|}{0.61} & \multicolumn{1}{c|}{1.34} & \multicolumn{1}{c|}{0.61} & \multicolumn{1}{c|}{0.31} & \multicolumn{1}{c|}{0.65} & \multicolumn{1}{c|}{0.11} & 
                              \multicolumn{1}{c|}{0.58} & \multicolumn{1}{c|}{1.37}  \\ \hline \hline
  \multirow{2}{*}{\makecell[c]{An image\\versus its\\gray-scale\\version}}          & \ours                       & \multicolumn{1}{c|}{0.81} & \multicolumn{1}{c|}{5.54} & \multicolumn{1}{c|}{0.88} & \multicolumn{1}{c|}{6.15}  & \multicolumn{1}{c|}{0.84} & \multicolumn{1}{c|}{7.71}  & \multicolumn{1}{c|}{0.83} & \multicolumn{1}{c|}{4.94} & \multicolumn{1}{c|}{0.86} & \multicolumn{1}{c|}{4.70} & \multicolumn{1}{c|}{0.82} & \multicolumn{1}{c|}{6.32}     \\ \cline{2-14} 
  & PSNR(dB)                       & \multicolumn{1}{c|}{19.96} & \multicolumn{1}{c|}{7.07} & \multicolumn{1}{c|}{17.67} & \multicolumn{1}{c|}{\underline{10.07}} & \multicolumn{1}{c|}{18.59} & \multicolumn{1}{c|}{5.71} & \multicolumn{1}{c|}{22.51} & \multicolumn{1}{c|}{5.79} & \multicolumn{1}{c|}{21.15} & \multicolumn{1}{c|}{\underline{9.57}} & 
  \multicolumn{1}{c|}{20.99} & \multicolumn{1}{c|}{7.60}             \\ \cline{2-14} 
  & MS-SSIM                       & \multicolumn{1}{c|}{0.93} & \multicolumn{1}{c|}{\textbf{15.25}} & \multicolumn{1}{c|}{0.90} & \multicolumn{1}{c|}{\textbf{29.00}} & \multicolumn{1}{c|}{0.89} & \multicolumn{1}{c|}{\textbf{20.95}} & \multicolumn{1}{c|}{0.95} & \multicolumn{1}{c|}{\textbf{12.75}} & \multicolumn{1}{c|}{0.94} & \multicolumn{1}{c|}{\textbf{18.91}} & 
  \multicolumn{1}{c|}{0.94} & \multicolumn{1}{c|}{\textbf{13.81}}     \\ \cline{2-14} 
                              & LPIPS                       & \multicolumn{1}{c|}{0.21} & \multicolumn{1}{c|}{\underline{7.23}} & \multicolumn{1}{c|}{0.22} & \multicolumn{1}{c|}{8.61} & \multicolumn{1}{c|}{0.20} & \multicolumn{1}{c|}{\underline{8.19}} & \multicolumn{1}{c|}{0.18} & \multicolumn{1}{c|}{\underline{7.63}} & \multicolumn{1}{c|}{0.12} & \multicolumn{1}{c|}{8.46} & 
                              \multicolumn{1}{c|}{0.17} & \multicolumn{1}{c|}{\underline{9.35}} \\ \hline 
  \end{tabular}}}
\end{table*}

\begin{itemize}
  \item \ours reaches the best performance in the 3 classes experiments (An image versus its inverse, its 90-degree rotation and its 180-degree rotation), indicating that \ours is more robust in measuring the semantic similarity when the image changes in structural flipping than the other 3 metrics.
  \item \ours achieves the second-best performance on most datasets in the experiment: an image versus random noise, while \lp performs the best. Since the \lp uses the mean pooling of feature distances from corresponding locations, it fails to precisely measure the global semantics of images but is sensitive to random noise.
  \item In the low-resolution and gray-scale experiments, \ms performs the best. It illustrates that \ms well performs in measuring the luminance, contrast, and structure of images according to its definition. On the other hand, its standard scores are higher than the other 3 metrics' scores due to the fact the standard deviation of \ms scores of any two images in a given dataset is extremely low.
\end{itemize}

In summary, we can conclude that \ours outperforms the \psnr, \ms, and \lp for the image similarity measurement at the global semantic level.
\ours shows stable performance in the same type of experiments over all testing datasets, suggesting that \ours is significantly efficient to characterize semantic aspects of images.
\looseness=-1

\subsection{Ablation Study Results}
\label{sec:ablation}

Table~\ref{tab:abalation} shows the results of the ablation study over the COCO dataset. For different forms of \ours have different distributions, we convert raw scores into standard scores by Equation (\ref{eq:standard_score}). We can find that, compared to the mean pooling version, the original \ours achieves better performance in most experiments, except in the random noise experiment. This is due to the tolerance of max pooling. Overall, the origin \ours can better reflect the semantic similarity.
\looseness=-1

\begin{table}[htbp]
  \renewcommand{\arraystretch}{1.5}
  \centering
  \caption{\label{tab:abalation}Ablation study results. Raw scores are converted into standard scores. The best results are shown in \textbf{bold}.}
  \begin{tabular}{|c|c|c|}
      \hline
       Experiment & $\ours_\mathrm{Origin}$ &  $\ours_\mathrm{Mean}$ \\
       \hline \hline
       GS & \textbf{6.32}  & 4.78 \\
       \hline
       I & \textbf{7.15}  & 5.10 \\
       \hline
       R90 & \textbf{5.03}  & 4.14\\
       \hline
       R180 & \textbf{4.46}  & 3.66 \\
       \hline
       RN & -0.36  & \textbf{-0.83} \\
       \hline
       LR & \textbf{3.64}  & 3.18 \\
       \hline
  \end{tabular}
\end{table}

\section{Conclusion and Discussion}

In this paper, we have proposed \ours, a novel semantic similarity evaluation metric for images based on the pre-trained image model ViT. 
To the authors' knowledge, \ours is the first practical image semantic similarity metric for SC.
We define the \ours and prove theoretically that, \ours has the properties of {\emph{symmetry, boundedness,}} and {\emph{normalization}}.
These important properties make \ours convenient and intuitive for implementation in image measurement.
Thanks to exploiting the attention mechanism of pre-trained models, \ours can automatically learn and extract essential semantics of the image.
Therefore, \ours exhibits an excellent ability to measure the semantics of images.
Furthermore, we evaluate the performance of \ours through 4 classes of experiments: (i) correlation with BERTScore through evaluation of image caption downstream CV task, (ii) evaluation in classical image communications, (iii) evaluation in image semantic communication~(including 4 semantic communication models~(NTSCC~\cite{dai2022nonlinear}, DeepJSCC-V~\cite{zhang2023predictive}, ADJSCC~\cite{Jialong} and DeepJSCC~\cite{Bourtsoulatze})), and (iv) evaluation in image semantic communication systems with semantic attack~(including 2 categories of semantic attack: semantic attack to mislead image classification task and semantic attack with image transforms). Experimental results demonstrate that ViTScore is robust and well-performed in evaluating the semantic similarity of images. Particularly, ViTScore outperforms the other 3 typical metrics in evaluating the image semantic changes by semantic attack, such as image inverse with GANs, indicating that ViTScore is an effective performance metric when deployed in SC scenarios. 
In addition, we provide an ablation study of ViTScore to demonstrate the structure of the ViTScore metric is valid.
Accordingly, we believe that \ours has a promising performance in the case when deployed in SC scenarios.
\looseness=-1

Currently, this work is just a first step towards SC for image semantic similarity evaluation. Advances in the foundation models, such as GPT-4 and Sora, and their applications in future communications can bring further performance improvements in image transmission. 
The metric \ours proposed in this paper can be easily extended along with foundation models.
Furthermore, a general framework design of SC for image and video will be our future work.
\looseness=-1

\section*{Acknowledgments}

\noindent This research was supported by the National Key R\&D Program of China~(No. 2021YFA1000500), the NSF of China~(No. 62276284), the Guangdong Basic and Applied Basic Research Foundation~(No.  2022A1515011355), Guangzhou Science and Technology Project~(No. 202201011699), Guizhou Provincial Science and Technology Projects~(No. 2022-259), Humanities and Social Science Research Project of Ministry of Education~(No. 18YJCZH006).
The authors would like to express their sincere gratitude to Prof. Xiao Ma for the valuable discussion and to anonymous reviewers for their critical comments on an earlier version of this paper that led to a much-improved version.

\bibliographystyle{IEEEtran}
\bibliography{ViTScore}

\end{document}